\documentclass{article}

\PassOptionsToPackage{numbers, compress, sort}{natbib}
 \usepackage[preprint]{neurips_2026}

\usepackage[utf8]{inputenc} 
\usepackage[T1]{fontenc}    
\usepackage[dvipsnames]{xcolor}  
\usepackage{hyperref}       
\hypersetup{
    colorlinks=true,
    linkcolor=MidnightBlue,
    citecolor=MidnightBlue,
    urlcolor=MidnightBlue,
    pdfborder={0 0 0},
    pdftitle={Neural Transport Nested Sampling},
    pdfauthor={David Yallup and Will Handley},
    pdfsubject={Neural samplers for Boltzmann distributions; nested sampling; flow matching},
    pdfkeywords={nested sampling, normalizing flows, flow matching, Boltzmann generators, Lennard-Jones, partition function, molecular sampling}
}
\usepackage{url}            
\usepackage{graphicx}       
\usepackage{subcaption}     
\usepackage{booktabs}       
\usepackage{multirow}       
\usepackage{amsfonts}       
\usepackage{amsmath}        
\usepackage{amssymb}        
\usepackage{amsthm}         
\newtheorem{proposition}{Proposition}
\usepackage{nicefrac}       
\usepackage{microtype}      
\usepackage[ruled,vlined]{algorithm2e}
\usepackage{cleveref}

\newcommand{\unc}[1]{\,{\color{gray}\scriptstyle\pm\,#1}}
\crefname{algocf}{Algorithm}{Algorithms}
\Crefname{algocf}{Algorithm}{Algorithms}
\crefname{subfigure}{Fig.}{Figs.}
\Crefname{subfigure}{Figure}{Figures}

\title{Neural Transport Nested Sampling}

  \author{%
    David Yallup\thanks{dy297@cam.ac.uk} \quad Will Handley \\
    Kavli Institute for Cosmology Cambridge \\
    Institute of Astronomy, University of Cambridge, \\
    Madingley Road, Cambridge, CB3 0HA, UK
  }

\begin{document}

\maketitle

\begin{abstract}

Sampling from Boltzmann distributions of molecular systems is an inference problem that has seen significant recent developments fuelled by advances in neural density estimation. We develop a novel sampling algorithm, Neural Transport Nested Sampling (NTNS), which combines the classical strengths of nested sampling with modern neural flow-based methods. NTNS uses a flow matching velocity as the drift in a Metropolis--Hastings corrected Langevin kernel inside a nested sampling outer loop, requiring only evaluations of the target energy function and providing scalable estimation of the full partition function of high-dimensional particle systems. We benchmark NTNS on challenging molecular sampling benchmarks, scaling up to Lennard--Jones clusters of 55 interacting particles, where it reduces both interatomic distance and energy Wasserstein errors to reference MCMC by over an order of magnitude relative to the strongest neural baselines at lower wall-clock cost. To our knowledge, NTNS is also the first neural sampler to return a calibrated, temperature resolved partition function estimate at this scale, recovering the phase structure across temperature from a single run.
  
\end{abstract}

\section{Introduction}\label{sec:intro}

The success of deep learning for protein structure prediction~\citep{jumper_highly_2021} has motivated a broader programme of bringing modern machine learning to bear on the simulation of molecular systems~\citep{batatia2023foundation,pmlr-v267-fu25h}. A natural and harder cousin of structure prediction, or energy calculation, is the molecular sampling problem: rather than a single configuration, downstream tasks in chemistry and biology---free energy estimation, conformational analysis, and phase behaviour---require samples from the Boltzmann ensemble together with the partition function that normalises it. Concretely, the target is
\begin{equation}\label{eq:target}
    p_\beta(x) = \frac{\exp(-\beta E(x))\,\pi(x)}{\mathcal{Z}(\beta)}\,,\qquad \mathcal{Z}(\beta) = \int \exp(-\beta E(x))\,\pi(x)\,\mathrm{d}x\,,
\end{equation}
where $E(x)$ is an evaluable energy function over particle positions, $\pi(x)$ is a reference prior, and $\beta$ is an inverse temperature; the physical target is $\beta=1$ and reweighting to other $\beta$ exposes the phase structure of the system. Even for idealised Lennard--Jones clusters, $p_\beta$ is highly multimodal: low-energy basins are separated by large barriers, phase transitions induce abrupt changes in typical configurations, and the number of metastable structures grows rapidly with system size~\citep{noe_boltzmann_2019,kohler_equivariant_2020,klein_equivariant_2023}. These pathologies are well known to be challenging for classical Markov-chain Monte Carlo (MCMC) methods~\citep{robert_monte_2004}. This paper asks whether the scalability of recent neural diffusion samplers can be combined with the normalising constant estimates and explicit correction mechanisms of particle methods.

Recent neural samplers attack this sample-free setting by learning time-dependent drifts or scores from energy evaluations rather than from target samples. Many path-integral and diffusion-based variants have emerged along these lines in recent years~\citep{zhang_path_2022,vargas2023denoising,vargas_transport_2023,berner_optimal_2023,richter_improved_2023}, demonstrating promising scaling on challenging problems. Within this family, sampling from Boltzmann distributions provides a particularly demanding test: Iterated Denoising Energy Matching (iDEM)~\citep{akhound-sadegh_iterated_2024} was the first method to scale to larger clusters, and subsequent methods based on adjoint matching~\citep{havens_adjoint_2025,liu_adjoint_2025,blessing2026bridge} have improved its efficiency. These methods define the current scalability frontier, but all of them differentiate the energy, their accuracy remains tied to the learned finite-time dynamics, and a calibrated, temperature resolved partition function is not an intrinsic output of the sampler.

Transport-augmented particle methods take a complementary route. Classical particle methods, AIS~\citep{neal_annealed_2001} and Sequential Monte Carlo (SMC) samplers~\citep{del_moral_sequential_2006,doucet_sequential_2001}, carry importance weights and yield consistent normalising constant estimates under standard assumptions. Flow-assisted variants such as AFT~\citep{arbel_annealed_2021}, CRAFT~\citep{matthews_continual_2022}, and FAB~\citep{midgley_flow_2022} make this framework adaptive by learning transport proposals along the annealing path. Their efficiency in high-dimensional molecular systems is limited by the quality of the learned proposal~\citep{grenioux_sampling_2023}, and they have not matched recent neural diffusion samplers on molecular benchmarks. Among particle methods, nested sampling~\citep{skilling_nested_2006} is distinguished by being a direct \emph{partition function} calculator: a single \emph{athermal} trajectory through energy contours estimates the cumulative density of states, allowing $\mathcal{Z}(\beta)$ to be evaluated at any temperature by post-hoc quadrature. This property has made nested sampling one of the established methods for Lennard--Jones thermodynamics in computational chemistry, where it has been used to recover phase diagrams for atomic clusters and bulk solids without prescribing a temperature schedule~\citep{partay_nested_2014,partay_nested_2021}.


We introduce Neural Transport Nested Sampling (NTNS), a nested sampling method whose constrained prior replacement kernel is accelerated by an online-trained flow matching drift inside a Metropolis-adjusted Langevin (MALA) kernel~\citep{roberts_exponential_1996}. Transport approximation error therefore affects efficiency rather than the stationary distribution of the replacement step. NTNS is gradient-free, the flow is trained by regression on the live points and the kernel touches the target only through energy evaluations, so the force never enters the method and no clipping or smoothing of the target is needed. On LJ-55, the highest-dimensional standard benchmark in this suite, NTNS achieves the closest agreement with reference MCMC of any neural sampler we tested, reducing both interatomic distance and energy Wasserstein errors by over an order of magnitude relative to the strongest prior neural baseline, at substantially lower wall-clock training cost (\cref{tab:w2_metrics,tab:timing}). On the smaller benchmarks (DW-4, LJ-13) NTNS is competitive with the best baselines on every metric and best on most. Unlike diffusion-style samplers, the same run also returns a temperature resolved nested sampling estimate of $\mathcal{Z}(\beta)$, enabling post-hoc reweighting and recovery of the LJ-55 phase diagram without additional simulations (\cref{fig:phase_diagram}).

We summarize our key contributions as follows:
\begin{enumerate}
    \renewcommand{\labelenumi}{(\roman{enumi})}
    \item We introduce Neural Transport Nested Sampling, a neural particle method that learns the geometry of each constrained prior online, addressing the central mutation bottleneck in nested sampling (\cref{sec:flow_mala}).
    \item We introduce the first sampling method to use a learned flow matching velocity directly as a local proposal, rather than integrating the induced flow. Metropolis--Hastings correction requires neither target-energy gradients nor flow-density evaluation, so learning error affects efficiency rather than the invariant distribution of the fixed kernel (\cref{prop:correctness}).
    \item We demonstrate consistent convergence on Lennard--Jones cluster benchmarks, comfortably outperforming representative state-of-the-art neural samplers. A matched classical nested sampling control isolates the learned proposal and confirms its benefit under the same outer construction (\cref{sec:experiments,sec:nss_control}).
    \item We show that a single NTNS run returns a calibrated, temperature resolved partition function and supports post-hoc sampling across temperatures, providing the first such neural sampling result at scale (\cref{fig:phase_diagram,sec:phase_appendix}).
\end{enumerate}

\section{Background}\label{sec:background}

NTNS sits at the intersection of two literatures that have largely developed separately: particle methods for estimating normalising constants, and neural transport methods for scalable sample generation. We work in the sample-free Boltzmann setting of \cref{eq:target}: the energy $E(x)$ can be evaluated, but samples from $p_\beta$ and the constant $\mathcal{Z}(\beta)$ are unavailable. When samples from the target \emph{are} available, this is a \emph{generative modelling} problem addressable by density estimation~\citep{dinh_density_2017,rezende_variational_2015,papamakarios_normalizing_2021,ho_denoising_2020,song_score-based_2020}. We cover the necessary background on nested sampling (\cref{sec:ns_background}), flow matching (\cref{sec:cfm}), and neural sampler baselines (\cref{sec:neural_baselines}).

\subsection{Nested Sampling}\label{sec:ns_background}

Nested sampling~\citep{skilling_nested_2006} is a particle method, widely used across the physical sciences~\citep{ashton_nested_2022} and with a long history of application to Lennard--Jones systems~\citep{partay_nested_2014,partay_nested_2021}, that computes the partition function by transforming the multi-dimensional integral into a one-dimensional integral over the \emph{prior volume}
\begin{equation}\label{eq:evidence}
    \mathcal{Z}(\beta) = \int e^{-\beta E(x)}\pi(x)\,\mathrm{d}x = \int_0^1 e^{-\beta E^*(X)}\,\mathrm{d}X, \qquad X(E^*) = \int_{E(x) < E^*} \pi(x)\,\mathrm{d}x,
\end{equation}
where $X(E^*)$ is the fraction of prior mass at energies below $E^*$, and $E^*(X)$ is its inverse. The algorithm maintains a set of $m$ live points $\{x_1, \ldots, x_m\}$ drawn from the prior. Indexing dead points $i = 1, 2, \ldots$ in order of removal, with energy $E_i$ and prior volume $X_i$ estimated from the dynamic live count order statistics~\citep{fowlie_nested_2021,yallup_nested_2026}, the live points at step $i$ are samples from the \emph{constrained prior},
\begin{equation}\label{eq:constrained_prior}
    \pi_i^*(x) = \frac{\mathbf{1}\{E(x) < E_i\}\,\pi(x)}{X_i},
\end{equation}
and the dead point trajectory yields a quadrature for the temperature dependent partition function
\begin{align}\label{eq:weights}
    \mathcal{Z}(\beta) \;\approx\; \sum_i w_i(\beta)\,, \qquad w_i(\beta) = e^{-\beta E_i}\,(X_{i-1} - X_i)\,.
\end{align}
In practice we batch $k$ removals before replenishing the live set, recovering the original Skilling algorithm at $k=1$. Crucially, $\beta$ enters only through the weights $w_i(\beta)$; the trajectory $\{(E_i, X_i)\}$ is temperature-independent, so a single run yields $\mathcal{Z}(\beta)$ and reweighted samples at every temperature.

The critical step is \cref{eq:constrained_prior}: sampling from the prior subject to a hard energy constraint. As the algorithm progresses the constrained region shrinks and can develop complex geometry (multiple modes, curved ridges, thin shells), making this step increasingly expensive. The hard boundary set by the energy threshold is also a unique challenge for parameterised neural surrogates, which produce smooth approximations.

Nested sampling is a particle method closely related to SMC~\citep{del_moral_sequential_2006,doucet_sequential_2001}. In the standard tempered SMC construction considered here, the path is indexed by an annealing temperature; nested sampling instead uses an adaptive, sorted energy schedule. Rather than softening the target with $\beta$, it maintains a hard sequence of energy thresholds $E^*$ determined online from the worst current particles, making nested sampling closest to \emph{adaptive} SMC~\citep{fearnhead_adaptive_2010}. Reweighting deletes those particles outright instead of multiplying their weights by an annealing factor, and the constrained prior replacement step plays the role of the SMC mutation kernel; see \Cref{fig:pt:paths} and \cref{sec:ns_vs_smc}.

\subsection{Flow Matching and Density Estimation}\label{sec:cfm}

Flow matching~\citep{lipman_flow_2022,tong_improving_2023} provides a simulation-free training objective for continuous normalising flows (CNFs)~\citep{chen_neural_2018}. Given samples $\mathbf{x}_0 \sim p_0$ from a base distribution and $\mathbf{x}_1 \sim q$ from a target (we use the independent coupling), the conditional flow matching (CFM) loss trains a velocity field $v_\phi$ by regressing onto conditional vector fields:
\begin{equation}\label{eq:cfm_loss}
    \mathcal{L}_\text{CFM}(\phi) = \mathbb{E}_{t \sim U[0,1],\, \mathbf{x}_1 \sim q,\, \mathbf{x}_0 \sim p_0} \left\| v_\phi(\mathbf{x}_t, t) - (\mathbf{x}_1 - \mathbf{x}_0) \right\|^2,
\end{equation}
where $\mathbf{x}_t = t\mathbf{x}_1 + (1-t)\mathbf{x}_0$ is the linear interpolant and $t \in [0,1]$ is the interpolation time, sampled uniformly during training (per \cref{eq:cfm_loss}).
This loss is equivalent to the marginal flow matching objective up to a constant and does not require simulating the ordinary differential equation (ODE) during training. In standard density estimation $q$ is the empirical data distribution; in NTNS $q$ is the current constrained prior $\pi_i^*$, represented by the live points.

The learned velocity field can be used in two ways~\citep{liu_flow_2022,albergo_stochastic_2025}. Deterministically, by integrating the ODE
\begin{equation}\label{eq:ode}
    \frac{\mathrm{d}\mathbf{x}}{\mathrm{d}t} = v_\phi(\mathbf{x}, t), \qquad \mathbf{x}(0) = \mathbf{x}_0 \sim p_0,
\end{equation}
transports base samples to target samples and gives a tractable density through the instantaneous change-of-variables formula
\begin{equation}\label{eq:change_of_vars}
    \log q_\phi(\mathbf{x}) = \log p_0(\mathbf{x}_0) - \int_0^1 \nabla \cdot v_\phi(\mathbf{x}(t), t)\,\mathrm{d}t,
\end{equation}
where $\mathbf{x}_0$ is obtained by integrating the flow backwards in time from $t=1$ to $t=0$ starting from $\mathbf{x}$. In practice this requires tens to hundreds of network evaluations together with divergence estimation. Through the stochastic interpolant and probability flow formulations, the flow matching velocity is connected to score-based and diffusion samplers~\citep{sohl-dickstein_deep_2015,ho_denoising_2020,song_score-based_2020}, as utilised in iDEM~\citep{akhound-sadegh_iterated_2024}; sampling via the corresponding reverse time stochastic differential equation (SDE) additionally requires a learned score in place of (or alongside) the velocity.

The constrained prior $\pi_i^*$ (\cref{eq:constrained_prior}) creates a particular challenge for flow matching: the target has hard boundaries set by the energy threshold, whereas neural networks produce smooth approximations (\Cref{fig:pt:compression}). We address this mismatch by using the flow not as a direct sampler but as a proposal inside a Metropolis--Hastings (MH) correction~\citep{metropolis_equation_1953,hastings_monte_1970} (\cref{sec:method}).

\begin{figure}
    \centering
    \begin{subfigure}[t]{0.64\columnwidth}
        \centering
        \includegraphics[width=\linewidth]{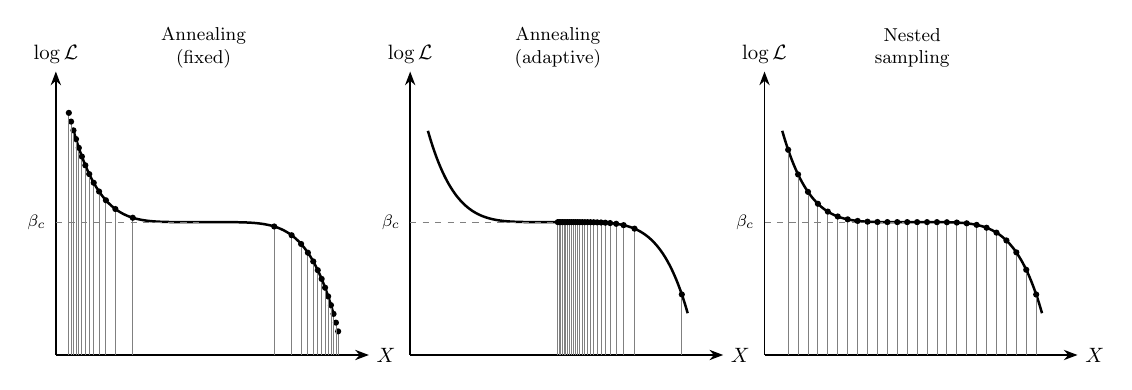}
        \caption{Annealing vs.\ nested sampling paths through energy landscape.}\label{fig:pt:paths}
    \end{subfigure}\hfill
    \begin{subfigure}[t]{0.34\columnwidth}
        \centering
        \includegraphics[width=\linewidth]{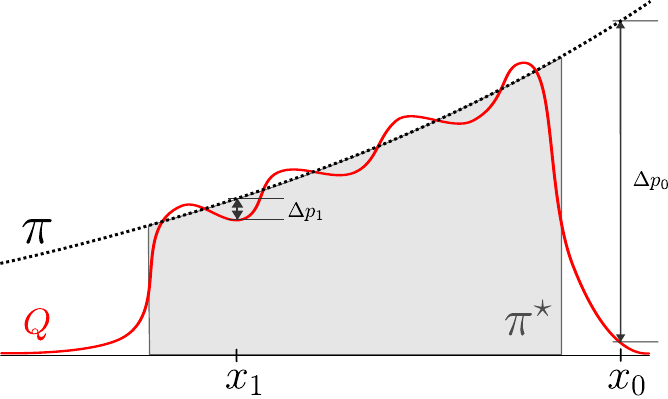}
        \caption{Nested sampling constrained prior target.}\label{fig:pt:compression}
    \end{subfigure}
    \caption{Nested sampling as adaptive neural transport. In \cref{fig:pt:paths}, annealed SMC follows a fixed temperature path and suffers weight collapse at phase transitions; nested sampling compresses the prior through energy contours determined online by the live points, giving $\mathcal{Z}(\beta)$ at every temperature simultaneously ($\log\mathcal{L} = -E$). In \cref{fig:pt:compression}, at each NS iteration the live points define the constrained prior $\pi_i^*$ (\cref{eq:constrained_prior}); a flow matching model provides a smooth proposal $Q_\phi$, used only inside an MH corrected kernel.}\label{fig:pt}
\end{figure}

\subsection{Related Prior Work}\label{sec:neural_baselines}

\textit{Neural samplers.} Diffusion and control-based samplers learn a neural parameterised drift that carries a tractable base distribution to the Boltzmann target~\citep{zhang_path_2022,vargas2023denoising,berner_optimal_2023,richter_improved_2023,vargas_transport_2023,geffner_langevin_2023,blessing_underdamped_2024,albergo_nets_2025,dern_energy-weighted_2026,pmlr-v235-blessing24a}. Relevant to the Lennard--Jones cluster setting, two patterns are worth highlighting: regressing a score network onto a Monte Carlo estimate of the score~\citep{akhound-sadegh_iterated_2024}, and regressing the drift onto the energy gradient at trajectory endpoints by adjoint matching~\citep{havens_adjoint_2025,liu_adjoint_2025,blessing2026bridge}. Without a correction step, sample quality is set by training error, and the normalising constant, where reported, is estimated by importance sampling the trained model, which carries no correction guarantee either. Other methods wrap the learned diffusion in a particle system with resampling~\citep{phillips_particle_2024,chen_sequential_2025}, or reweighting of non-equilibrium trajectories~\citep{albergo_nets_2025}. NTNS differs from this setting as the learned dynamics are never simulated, but used as a fixed guiding force to mutate particles efficiently, with only an energy-based Metropolis correction.

\textit{Particle methods with learned proposals.} AIS~\citep{neal_annealed_2001} and SMC samplers~\citep{del_moral_sequential_2006,doucet_sequential_2001} carry importance weights and yield consistent normalising constants. Boltzmann generators~\citep{noe_boltzmann_2019}, stochastic normalising flows~\citep{wu_stochastic_2020}, AFT~\citep{arbel_annealed_2021}, CRAFT~\citep{matthews_continual_2022}, FAB~\citep{midgley_flow_2022} and preconditioned Monte Carlo~\citep{karamanis_accelerating_2022} add learned transport proposals, as do flow-based proposals within nested sampling~\citep{moss_accelerated_2020,williams_nested_2021,williams_importance_2023,lange_nautilus_2023}. These usually use the flow as a global importance proposal, whose acceptance degrades with dimension below the largest scales we target here~\citep{grenioux_sampling_2023}. NTNS is distinct in using the flow for local mutation, which is particularly suited to the hard constrained targets of nested sampling.

\textit{Flow-augmented MCMC.} A learned flow has also been used to accelerate standard MCMC on a fixed target, either by reparameterising the target and running a classical kernel in the flow's latent space~\citep{hoffman2019neutralizingbadgeometryhamiltonian,cabezas_transport_2023}, or as a global independence proposal interleaved with local Langevin moves~\citep{gabrie_adaptive_2022,wong_flowmc_2023}. Both need an invertible flow with a tractable density, and in both the flow is fitted to one fixed target and the samples are the chain states. NTNS uses the flow in a third way, as the drift of a local Langevin kernel, which needs neither the inverse map nor the density. Its flow is a transport along the whole nested sampling path rather than a fit to one target, which keeps open the neural sampler goal of amortising a transport across systems~\citep{klein_timewarp_2023,klein_transferable_2024}.

\textit{Equivariant architectures.} For an $n$-body particle system in $\mathbb{R}^3$ the energy is invariant under the Euclidean group $\mathrm{E}(3)$ and the symmetric group $\mathbb{S}_n$ on particles~\citep{kohler_equivariant_2020}. Because $\mathrm{E}(3) \times \mathbb{S}_n$ generates entire families of symmetry-related configurations, classical samplers mix poorly unless one quotients out these redundancies---exactly the inductive bias provided by equivariant architectures. E(3)-equivariant graph neural networks (EGNNs) were introduced by \citet{satorras_en_2021}; SE(3)-equivariant extensions have since been applied to a variety of protein and molecular sampling problems~\citep{klein_equivariant_2023,midgley_se3_2023,bose_se3-stochastic_2023,yim_fast_2023}. We adopt an EGNN backbone in line with~\citep{akhound-sadegh_iterated_2024,havens_adjoint_2025}; details in~\cref{sec:egnn_details}.

\section{Method}\label{sec:method}

We now present NTNS, which targets the constrained prior sequence $\{\pi_i^*\}$ of \cref{eq:constrained_prior} by retaining the standard outer nested sampling kernel and replacing only the constrained prior replacement step with a learned MH kernel. Following the standard sequential Monte Carlo decomposition~\citep{chopin_introduction_2020}, each outer iteration consists of: (i) \emph{reweight}: delete the $k$ highest-energy live points, fix the threshold $E^*$, and record them as dead with weights $w_i(\beta)$ from \cref{eq:weights}; (ii) \emph{resample}: draw $k$ parents uniformly from the $m-k$ survivors; (iii) \emph{mutate}: apply a constrained update kernel $\mathcal{K}_{E^*}$ to each parent; and (iv) \emph{replace}: insert the mutated samples into the live set. NTNS modifies only the mutation kernel; the outer loop is stated in full in \cref{alg:ns_outer}.

The mutation kernel $\mathcal{K}_{E^*}$ is constructed at every NS iteration by training a flow matching model $Q_\phi$ that transports a mean-subtracted multivariate normal base to the survivors remaining after the reweight step. The learned velocity field can be used in two ways: as a direct density proposal inside an independent Metropolis--Hastings kernel (\cref{sec:irmh}), or as a learned drift inside an MH corrected Langevin kernel. We focus on the latter in the main text, finding it the most effective and scalable construction. Both choices preserve $\pi_i^*$-invariance via the MH correction, but the Langevin construction admits a cheap correction that avoids the ODE integration making flow-density evaluation prohibitive in high dimensions.

\subsection{Neural drift Monte Carlo kernels}\label{sec:flow_mala}
The learned flow $Q_\phi$ supplies two MH-correct mutation kernels for $\pi_i^*$: a density-based independent proposal that uses the flow density $Q_\phi(x)$ directly, and a velocity-based local proposal that uses only the velocity field $v_\phi(x, t)$ as a Langevin drift. A useful lens on many learned samplers is that they neutralise pathological target geometry by mapping the target through a learned transport from a simple base measure~\citep{hoffman2019neutralizingbadgeometryhamiltonian,albergo_nets_2025}.

Within the ODE-based flow family of \cref{eq:ode}, prior work has used the deterministic transport directly as a proposal density. With a continuous normalising flow~\citep{chen_neural_2018} the divergence integral in \cref{eq:change_of_vars} costs $\mathcal{O}(d)$ vector-Jacobian products per integration step, prohibitive at the dimensions we target. We implement this independent MH (IMH) kernel for completeness and analyse its scaling, alongside cheaper alternatives, in \cref{sec:irmh}.

We propose an alternative that avoids flow density evaluation entirely. We train a flow $Q_\phi$ on the current live points, which are samples from $\pi_i^*$. The flow velocity $v_\phi(x, t)$ transports samples from a Gaussian base toward $\pi_i^*$; at intermediate times $v_\phi(x, t^*)$ provides an informed drift towards regions of high $\pi_i^*$ density (\cref{fig:pt:compression}). As the kernel mutates particles already drawn from the constrained target, no transport from the flow prior is needed and a fixed late time is informative for this target; we use $t^* = 0.5$, which smooths the otherwise hard edge, and sample quality is only mildly sensitive to this choice (\cref{sec:ablations}).

We use this learned velocity as the drift in a MALA-style update~\citep{roberts_exponential_1996}, with $v_\phi(x, t^*)$ in place of the usual $\nabla \log \pi$ score:
\begin{equation}\label{eq:mala_proposal}
    x_{t+1} = x_t + \epsilon \, v_\phi(x_t, t^*) + \sqrt{2\epsilon}\,\xi, \qquad \xi \sim \mathcal{N}(0, I).
\end{equation}
The MH acceptance probability is
\begin{equation}\label{eq:mala_accept}
    \alpha = \min\left(1,\; \frac{\pi(x_{t+1})\, q(x_t \mid x_{t+1})}{\pi(x_t)\, q(x_{t+1} \mid x_t)}\right) \cdot \mathbf{1}[E(x_{t+1}) < E^*].
\end{equation}
At first glance this appears to require evaluating a density $q$---but crucially, $q$ here is the \emph{Gaussian transition density}, not the flow density $Q_\phi$:
\begin{equation}\label{eq:proposal_density}
    q(x_{t+1} \mid x_t) = \mathcal{N}(x_{t+1}; x_t + \epsilon v_\phi(x_t, t^*), 2\epsilon I).
\end{equation}
Evaluating $q(x_{t+1} \mid x_t)$ requires only a forward pass of $v_\phi$ to compute the mean; the reverse proposal $q(x_t \mid x_{t+1})$ requires one further pass at $x_{t+1}$, and the prior ratio is closed-form. The current state drift can be cached in a Markov Chain construction reducing \cref{eq:proposal_density} to a single network evaluation to evaluate the transition probability. The MALA-style construction therefore avoids the backward ODE that IMH requires for $Q_\phi$. For fixed $\phi$ and $t^*$ the kernel preserves $\pi_i^*$-invariance by standard MH detailed balance (\cref{prop:correctness}); full pseudocode in \cref{alg:flow_mala}; further details in \cref{sec:appendix}.

\subsubsection{Learned kernels within nested sampling}

This construction also addresses a gap in the nested sampling literature. In \cref{eq:constrained_prior} the target is the constrained prior, for which the dominant replacement strategy is slice sampling~\citep{neal_slice_2003,handley_polychord_2015,yallup_nested_2026}. Bringing gradient-based moves to this setting is non-trivial because the exact target has a hard energy constraint: away from the contour, the score is just $\nabla \log \pi(x)$, which reflects the unconstrained prior rather than the geometry of the constrained region, while the boundary contribution is singular. Existing approaches therefore require reflective dynamics to respect the constraint~\citep{pmlr-v235-lemos24a}, but these can be difficult to scale robustly in practice~\citep{PhysRevE.111.045308}. Our approach instead uses a learned surrogate $Q_\phi$ to provide a smooth, gradient-like signal adapted to the support of $\pi_i^*$, while retaining exactness through a Metropolis--Hastings correction. This is particularly valuable in high dimensions, where proposals that do not exploit gradient information suffer much poorer scaling---existing theory for hybrid slice samplers gives mixing times that grow polynomially in $d$~\citep{power2025weakpoincareinequalitycomparisons}.

\begin{proposition}[Invariant distribution of the constrained prior mutation kernel]\label{prop:correctness}
The Markov kernel defined by \cref{eq:mala_proposal,eq:mala_accept} leaves the constrained prior $\pi_i^*$~(\cref{eq:constrained_prior}) invariant, regardless of the quality of the flow $Q_\phi$.
\end{proposition}

\noindent The proof is a standard detailed-balance argument for Metropolis--Hastings kernels; see \cref{sec:proof_correctness}. The proposition establishes the invariant distribution of the mutation kernel used inside each NS iteration. As with all finite-move NS/SMC implementations~\citep{chopin_introduction_2020}, the full algorithm additionally depends on the practical adequacy of the finite mutation budget. As in SMC samplers with Markov rejuvenation kernels, including flow-assisted methods such as AFT~\citep{arbel_annealed_2021} and CRAFT~\citep{matthews_continual_2022}, the practical algorithm does not run each chain to exact equilibrium; instead, a finite number of corrected Markov moves is used to produce an approximate sample from the next particle population. The role of the MH correction is to ensure that transport error cannot change the invariant constrained prior, while the number of inner steps $T$ controls the remaining finite-mixing error. We therefore treat $T$ as a numerical accuracy parameter and assess stability by varying it (\cref{sec:ablations}), analogous to step size or particle number checks in SMC.

\subsection{Training and adaptation}\label{sec:practical_ntns}
We embed the learned drift MALA kernel in the outer NS loop (\Cref{alg:ns_outer}) with a few practical choices. The initial live set is produced by drawing $100\,m$ samples from the prior and keeping the $m$ lowest-energy---a standard trick that yields a lower-variance training set for the first flow than raw prior draws, and is itself just a nested sampling replacement step at $\sim$1\% rejection efficiency. Only the $m-k$ surviving live points are available as training data at each subsequent iteration; the live set thereby serves as the sample buffer, holding samples of the current constrained target rather than the noised model samples of the replay buffer used in practice by many neural samplers~\citep{akhound-sadegh_iterated_2024}. Over hundreds of iterations training a flow from scratch each time is prohibitive (\Cref{sec:ablations}); we instead \emph{warm start} from the previous iteration's parameters and briefly fine-tune, exploiting the small change between $\pi_{i-1}^*$ and $\pi_i^*$. Network training nevertheless dominates runtime. We adapt $\epsilon$ by stochastic step size adaptation~\citep{robbins_stochastic_1951} on the empirical acceptance pooled across the $k$ chains and $T$ inner steps, with $t^* = 0.5$ throughout. If acceptance collapses we retrain from scratch on the current live points with $\epsilon$ reset (\Cref{fig:lj55_seeds_diagnostics}). NTNS is sequential and adaptive rather than a single target equilibrium sampler: each MH kernel targets $\pi_i^*$ in equilibrium, while the outer loop traverses $\{\pi_i^*\}$ until standard NS stopping criteria are satisfied~\citep{skilling_nested_2006} (\Cref{sec:convergence}).

\section{Experiments}\label{sec:experiments}

A standard set of test problems has emerged for neural samplers targeting Boltzmann distributions: the double-well potential DW-4 (8D), the Lennard-Jones cluster LJ-13 (39D), and LJ-55 (165D)~\citep{akhound-sadegh_iterated_2024,havens_adjoint_2025}. These targets exhibit multimodal energy landscapes and scale to challenging dimensionality. They model four particles interacting in 2D (DW-4), and 13- and 55-particle systems interacting under the Lennard-Jones potential. Despite the simplicity of the energy functions, these are challenging benchmarks; full definitions are given in \cref{sec:target_models}.
Experiments are implemented in the \texttt{jax} framework~\citep{bradbury_jax_2018}, using the \texttt{blackjax} library as the backbone for the stochastic sampling~\citep{cabezas_blackjax_2024}. The nested sampling outer loop then follows the patterns described in~\citep{yallup_nested_2026}.
All experiments are run on a single NVIDIA GH200 Grace Hopper system~\citep{mcintosh-smith_isambard-ai_2024}, with 96\,GB of GPU memory.

We follow standard choices for the EGNN architecture~\citep{satorras_en_2021} when applied to these problems~\citep{akhound-sadegh_iterated_2024,havens_adjoint_2025}. NTNS hyperparameters are detailed in~\cref{sec:hyperparameters}. We compare against iDEM~\citep{akhound-sadegh_iterated_2024}, Adjoint Sampling (AS)~\citep{havens_adjoint_2025} and the Adjoint Schr\"odinger Bridge Sampler (ASBS)~\citep{liu_adjoint_2025}, the neural samplers with reported results at LJ-55 scale. We omit learned SMC-style samplers (e.g.\ CRAFT, FAB) as these have not matched diffusion samplers on LJ-55 in prior work~\citep{akhound-sadegh_iterated_2024}. We additionally include a direct classical nested sampling control based on vectorized Nested Slice Sampling (NSS)~\citep{yallup_nested_2026}, holding the outer schedule fixed while replacing the learned mutation kernel (\cref{sec:nss_control}); we leave a broader survey of hand-designed constrained kernels to future work. We summarise the performance of each sampler in \cref{tab:w2_metrics}, using the distance metrics defined in~\cref{sec:metrics} computed against long-run parallel reference MCMC chains~\citep{klein_equivariant_2023,akhound-sadegh_iterated_2024}. Interatomic distance and energy are sensible summary statistics; the interatomic distance in particular emphasises the multimodality of these targets. For the Lennard-Jones systems we display these plots in~\cref{fig:lj_diagnostics}, and include the DW-4 equivalent in~\cref{fig:dw4_diagnostics}.

NTNS achieves the best $r$-$W_2$ on every system and the best $E$-$W_2$ on the two Lennard--Jones systems; on the small DW-4 target the baselines are within statistical error on $E$-$W_2$ and the choice of sampler matters little. The clearest gap opens on the highest-dimensional benchmark. We complement this with the runtimes in~\cref{tab:timing}: NTNS is significantly faster in walltime to train, although AS-based methods require fewer target evaluations. We focus particularly on the LJ-55 system, where, to our knowledge, NTNS is the first method to converge consistently across seeds. We support this with five independent seeds of iDEM and NTNS (\cref{fig:seeds}) and find throughout that all methods can suffer training divergences (we observed instabilities in AS and ASBS as well); NTNS successfully mitigates these with a restart from the current live set. On the most challenging system, LJ-55, NTNS gives over an order of magnitude reduction in $r$-W$_2$ and $E$-W$_2$ and converges consistently in approximately 5.5 hours.

\begin{figure}[t]
    \centering
    \begin{subfigure}[t]{\columnwidth}
        \centering
        \includegraphics[width=\linewidth]{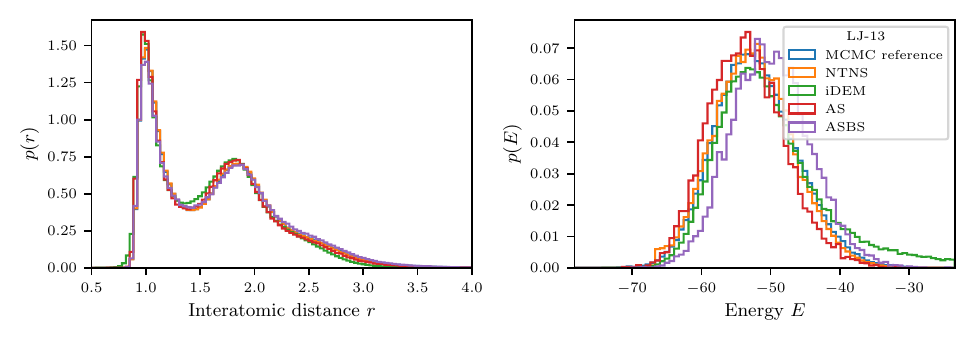}
        \caption{LJ-13.}\label{fig:lj_diagnostics:lj13}
    \end{subfigure}\\
    \begin{subfigure}[t]{\columnwidth}
        \centering
        \includegraphics[width=\linewidth]{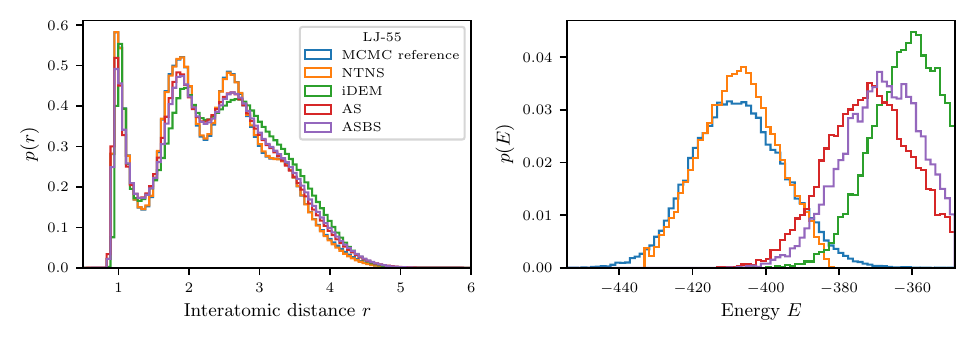}
        \caption{LJ-55.}\label{fig:lj_diagnostics:lj55}
    \end{subfigure}
    \caption{Posterior diagnostics for LJ-13 (\cref{fig:lj_diagnostics:lj13}) and LJ-55 (\cref{fig:lj_diagnostics:lj55}). For each system, the left panel shows the pooled pairwise interatomic distance distribution and the right panel the energy distribution for NTNS, iDEM, AS and ASBS, compared against long-run reference MCMC samples. NTNS closely matches the MCMC reference on both summaries across system sizes, while the neural sampler baselines place their energy mass above the reference, with a shift and tail that worsen with system size.}\label{fig:lj_diagnostics}
\end{figure}

\begin{table}[h]
\centering
\caption{Wasserstein-2 distances to MCMC reference samples for interatomic distances ($r$-W$_2$) and energy ($E$-W$_2$) for NTNS (ours) and other neural samplers. Bold text indicates the best performance for each metric. For LJ-55, the NTNS and iDEM entries are median-seed runs drawn from the five-seed sweep detailed in \cref{sec:convergence}.}
\label{tab:w2_metrics}
\resizebox{\textwidth}{!}{%
\begin{tabular}{lcccccc}
\toprule
& \multicolumn{2}{c}{DW-4} & \multicolumn{2}{c}{LJ-13} & \multicolumn{2}{c}{LJ-55} \\
\cmidrule(lr){2-3} \cmidrule(lr){4-5} \cmidrule(lr){6-7}
Method & $r$-W$_2$ $\downarrow$ & $E$-W$_2$ $\downarrow$ & $r$-W$_2$ $\downarrow$ & $E$-W$_2$ $\downarrow$ & $r$-W$_2$
$\downarrow$ & $E$-W$_2$ $\downarrow$ \\
\midrule
iDEM & $0.296 \unc{0.043}$ & $\mathbf{0.226} \unc{0.072}$ & $0.102 \unc{0.007}$ & $55.7 \unc{37.0}$ & $0.135 \unc{0.001}$ & $48.4 \unc{0.42}$ \\
AS   & $0.212 \unc{0.033}$ & $0.445 \unc{0.086}$ & $0.044 \unc{0.005}$ & $1.39 \unc{0.24}$ & $0.060 \unc{0.002}$ & $36.4 \unc{0.40}$ \\
ASBS & $0.172 \unc{0.044}$ & $0.304 \unc{0.116}$ & $0.029 \unc{0.005}$ & $2.24 \unc{0.22}$ & $0.082 \unc{0.001}$ & $40.2 \unc{0.50}$ \\
\midrule
NTNS & $\mathbf{0.154} \unc{0.034}$ & $0.286 \unc{0.086}$ & $\mathbf{0.018} \unc{0.008}$ & $\mathbf{0.58} \unc{0.15}$ & $\mathbf{0.008} \unc{0.001}$ & $\mathbf{2.68} \unc{0.17}$ \\
\bottomrule
\end{tabular}%
}
\end{table}

\subsection{Partition function calculation}
\begin{figure}[t]
    \centering
    \begin{subfigure}[t]{0.5\columnwidth}
        \centering
        \includegraphics[width=\linewidth]{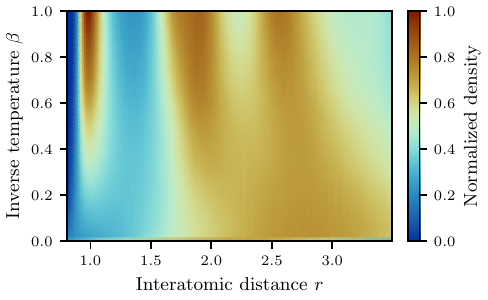}
        \caption{Phase diagram from a single NTNS run on LJ-55, showing sample distributions across inverse temperatures $\beta$.}\label{fig:phase_diagram:beta}
    \end{subfigure}\hfill
    \begin{subfigure}[t]{0.44\columnwidth}
        \centering
        \includegraphics[width=\linewidth]{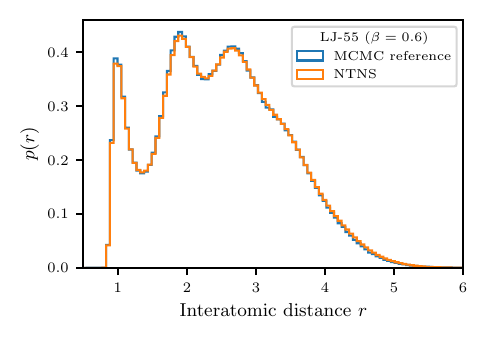}
        \caption{Validation at $\beta = 0.6$ against NUTS reference, confirming accurate post-hoc reweighting.}\label{fig:phase_diagram:validation}
    \end{subfigure}
    \caption{Temperature reweighting from a single LJ-55 NTNS run: reweighted sample distributions across values of the inverse temperature $\beta$ expose the phase structure (\cref{fig:phase_diagram:beta}), and validation at $\beta = 0.6$ against an independent NUTS reference confirms post-hoc recovery of intermediate temperature targets (\cref{fig:phase_diagram:validation}).}\label{fig:phase_diagram}
\end{figure}

A distinctive capability of nested sampling is recovering the temperature dependent partition function $\mathcal{Z}(\beta)$ from a single run via post-hoc reweighting of the dead points (\cref{eq:weights}). \Cref{fig:phase_diagram:beta} shows the resulting LJ-55 sample distribution across inverse temperatures $\beta \in [0,1]$, exposing the phase structure of the system without separate simulations at each temperature. We validate the reweighting against NUTS~\citep{hoffman_no-u-turn_2014} samples at $\beta = 0.6$ (\cref{fig:phase_diagram:validation}), confirming that reweighted NTNS samples reproduce the intermediate temperature target. As an independent cross-check, thermodynamic integration from NUTS samples gives $\log \mathcal{Z}(1.0) - \log \mathcal{Z}(0.6) = 148.8 \pm 0.4$, in agreement with NTNS's $147.5 \pm 0.4$ (methodology and further analysis in \cref{sec:phase_appendix}). The LJ-55 run produced $486{,}400$ weighted samples in total ($204{,}800$ from the initial prior-rejection oversample of the live set and $281{,}600$ from the subsequent flow-kernel iterations), yielding effective sample sizes of $54{,}705$ at $\beta = 0.6$ and $59{,}282$ at $\beta = 1.0$.

\section{Discussion}\label{sec:discussion}
The key design move in NTNS is to replace an IMH-style use of flows as global proposals, common in prior transport-augmented particle methods, with a more local learned drift inspired by diffusion samplers. This improves scalability, while the NTNS particle system still supplies the global exploration needed for good mixing. In that sense, NTNS occupies a useful middle ground between local corrected MCMC and population-based global search. The matched nested slice sampling control (\cref{sec:nss_control}) confirms that the gain over classical nested sampling comes from this kernel rather than from the outer loop.

Positioning NTNS as a localized approach also clarifies our comparison with AS and ASBS. Those methods are more energy-evaluation efficient than NTNS, but target energies in our current benchmarks are cheap to evaluate, rendering this difference negligible in practice. For real molecular systems with neural-network-learned energy functions~\citep{pmlr-v267-fu25h,batatia2023foundation}, however, energy evaluation can be considerably heavier. Adjoint matching uses fewer energy evaluations by design (gradients only at trajectory endpoints), making its integration with our inner kernel an important direction for further study. NTNS, by contrast, never evaluates the force. The neural sampler baselines regress onto $\nabla E$, and on a stiff potential the divergent short-range force must be clipped to keep that target finite~\citep{akhound-sadegh_iterated_2024,havens_adjoint_2025}; NTNS sees the target only through energy comparisons against a threshold, so no clipping or smoothing of the energy is needed, an advantage for stiff or non-differentiable energy functions.

Because our target energies are cheap, our current settings allocate the bulk of compute to flow training, but this balance is a tunable allocation rather than a fixed limitation. Modulating parameters such as flow epochs, the number of chains, inner steps $T$, and warm-start aggressiveness actively shifts the computational bottleneck to accommodate more expensive targets. Currently, we rely on a basic warm-start heuristic, and our evaluation cost is also driven by the need to scale the number of inner-kernel steps with the problem dimension. The ablations in \cref{sec:ablations} suggest that both choices may be conservative. More efficient use of particle history, online tuning of the inner steps, and methods that exploit repeated steps more directly are promising avenues for future research~\citep{dau_waste-free_2022}.

A natural next step is to integrate these ideas into a fully interacting SMC-type system, as an extension of CRAFT~\citep{matthews_continual_2022}. A longer-term goal shared with neural samplers is amortisation, reusing a transport learned on one system for another~\citep{klein_timewarp_2023,klein_transferable_2024}. NTNS, as a purely neural sampler that iteratively trains its proposal, produces a flow at termination that can be simulated as a neural sampler, and hence follows the neural sampler pattern rather than flow-augmented MCMC. Other interesting avenues are to explore the samples as intermediate checkpoints between modes of different atomic configurations, as a flow-based estimator of free energy differences~\citep{holdijk2023stochastic}.

\section{Conclusion}\label{sec:conclusion}
We have presented Neural Transport Nested Sampling (NTNS), which accelerates the constrained prior replacement step in nested sampling by using a learned drift in a MALA kernel. This lets us integrate expressive conditional flow matching models into nested sampling while avoiding expensive density evaluation. This work combines three successful design patterns: (i) the nested sampling algorithm, with provenance for Lennard--Jones clusters~\citep{partay2010efficient} but reliant on hard-to-design constrained prior updates; (ii) flow transport within SMC-type constructions, which provides learned adaptation but has historically used global proposals that have trouble scaling; and (iii) scalable learned-drift kernels, which scale well but lack the correction guarantees of (ii). NTNS combines these in such a manner that each weakness is mitigated by another.

We demonstrate NTNS scaling to systems of 55 interacting particles, giving significantly improved accuracy over other neural samplers, and being both economical in energy evaluations and runtime. We use a simulation-free objective to train the surrogate drift, and at sampling time each kernel step requires only a single forward pass of the velocity network and a single target-energy evaluation.

Beyond the molecular dynamics setting, NTNS is relevant to a broader class of inference problems where scalable sampling under repeated target evaluation is needed. It also produces a full particle history along the characteristic and unique nested sampling path, which can be reweighted athermally and may support downstream uses beyond sampling itself. The main open challenges are scaling to larger particle systems and handling more complex and costly energy functions.

\subsubsection*{Code and Data Availability}
The algorithm code and example scripts to reproduce the experiments in this paper are available at \url{https://github.com/yallup/ntns}.

\subsubsection*{AI and LLM tool use}
We used OpenAI GPT-5.4 to refine portions of the draft, and Claude Opus 4.7 was used in refining the code. The authors take full responsibility for the final content.

\subsubsection*{Broader Impact}
Efficient sampling of energy landscape has potential impact in a variety of high impact scientific domains, such as drug discovery and research into novel materials. Whilst we work on synthetic energy functions, directionally we envisage downstream impact in such applications.

\subsubsection*{Acknowledgements}
This work was supported by the UKRI Frontier Research Guarantee [EP/X035344/1]. The authors acknowledge the use of resources provided by the Isambard-AI National AI Research Resource (AIRR). Isambard-AI is operated by the University of Bristol and is funded by the UK Government’s Department for Science, Innovation and Technology (DSIT) via UK Research and Innovation; and the Science and Technology Facilities Council [ST/AIRR/I-A-I/1023].

{\small
\bibliographystyle{unsrtnat}
\bibliography{references}
}


\newpage
\appendix
\section{Target Distributions}\label{sec:target_models}

For each $N$-particle system below we work on the zero centre-of-mass affine subspace
\begin{equation}\label{eq:com_subspace}
    S = \{x \in \mathbb{R}^{Nd} : \textstyle\sum_{i=1}^N x_i = 0\},
\end{equation}
of dimension $(N{-}1)d$. The energy $E$ depends only on pairwise distances $r_{ij}$ and is therefore invariant under translations of $\mathbb{R}^{Nd}$; restricting to $S$ quotients out this translational symmetry. The prior in each case is the standard Gaussian on $S$, equivalently $\mathcal{N}(0, I_{Nd})$ restricted to $S$; together with $E$ this gives a target invariant under $\mathrm{O}(d) \times \mathbb{S}_N$ (rotation/reflection of $\mathbb{R}^d$ and permutation of identical particles). Ambient dimensions stated for each system below refer to $\mathbb{R}^{Nd}$.

\subsection{Double-Well Potential (DW-4)}\label{sec:dw4_model}

The DW-4 system consists of $N=4$ particles in $\mathbb{R}^2$, giving $d = 8$ dimensions.
The energy function is a sum over pairwise interactions:
\begin{equation}
    E(x) = \sum_{i < j} V_\text{DW}(r_{ij}), \qquad V_\text{DW}(r) = a(r - r_0)^4 + b(r - r_0)^2 + c,
\end{equation}
where $r_{ij} = \|x_i - x_j\|$ and we use parameters $a = 0.9$, $b = -4$, $c = 0$, $r_0 = 4$~\citep{kohler_equivariant_2020,akhound-sadegh_iterated_2024}.
The prior is the standard Gaussian on the zero-CoM subspace $S$ (\cref{eq:com_subspace}).
This target exhibits bimodal structure in the interatomic distance distribution.

\subsection{Lennard-Jones Cluster (LJ-13)}\label{sec:lj13_model}

The LJ-13 system consists of $N = 13$ particles in $\mathbb{R}^3$, giving $d = 39$ dimensions.
The energy is the standard Lennard-Jones potential:
\begin{equation}
    E(x) = 4\varepsilon \sum_{i < j} \left[\left(\frac{\sigma}{r_{ij}}\right)^{12} - \left(\frac{\sigma}{r_{ij}}\right)^{6}\right],
\end{equation}
where $r_{ij} = \|x_i - x_j\|$ and we use reduced units $\varepsilon = \sigma = 1$.
The prior is the standard Gaussian on the zero-CoM subspace $S$ (\cref{eq:com_subspace}).

\subsection{Lennard-Jones Cluster (LJ-55)}\label{sec:lj55_model}

The LJ-55 system consists of $N = 55$ particles in $\mathbb{R}^3$, giving $d = 165$ dimensions.
The energy function is identical to LJ-13 but at significantly larger scale.
This system exhibits a solid-liquid phase transition and represents a challenging benchmark for sampling methods.
The prior is the standard Gaussian on the zero-CoM subspace $S$ (\cref{eq:com_subspace}).

\section{Reference MCMC Data}\label{sec:reference_data}

We use reference MCMC samples from~\citet{klein_equivariant_2023} as redistributed in the iDEM codebase~\citep{akhound-sadegh_iterated_2024} for diagnostic comparisons. These consist of:
\begin{itemize}
    \item DW-4: $100{,}000$ training, $10{,}000$ validation, and $10{,}000$ test configurations, each of shape $(4, 2)$.
    \item LJ-13: $100{,}000$ training, $10{,}000$ validation, and $10{,}000$ test configurations, each of shape $(13, 3)$.
    \item LJ-55: $10{,}000$ training, $10{,}000$ validation, and $10{,}000$ test configurations, each of shape $(55, 3)$.
\end{itemize}
These samples were generated by long MCMC runs at the target temperature and are used to validate our posterior samples via pairwise interatomic distance and energy histograms. Note that NTNS does not use these samples for training---they serve only as independent ground truth for comparison.

\section{Metrics}\label{sec:metrics}

We evaluate sample quality with the 2-Wasserstein distance against MCMC reference samples. For empirical measures $\mu, \nu$ on $\mathbb{R}^d$ the 2-Wasserstein distance under Euclidean ground cost is
\begin{equation}\label{eq:w2}
\mathcal{W}_2(\mu,\nu) = \left( \inf_{\gamma \in \Gamma(\mu,\nu)} \int \|x-y\|_2^2 \,\mathrm{d}\gamma(x,y) \right)^{\!1/2},
\end{equation}
where $\Gamma(\mu,\nu)$ denotes couplings with marginals $\mu$ and $\nu$. We report two scalar summaries. We do not report the geometric $\mathcal{W}_2$ on raw configurations $x \in \mathbb{R}^{Nd}$: even with the POT~\citep{flamary_pot_2021} solver, particle-permutation symmetry in flattened Euclidean space produces a transport plan whose value is dominated by solver inefficiency and finite-sample noise rather than the differences between samplers we wish to measure. We instead use $r$-$\mathcal{W}_2$, the 1D $\mathcal{W}_2$ between empirical distributions of the pooled pairwise interatomic distances
\begin{equation}\label{eq:r_def}
r(x) = \{\,\|x_i - x_j\|_2 \,:\, 1 \le i < j \le N\,\},
\end{equation}
which are E$(n)$- and permutation-invariant by construction and capture local structural information without an alignment step. The second is $E$-$\mathcal{W}_2$, the 1D $\mathcal{W}_2$ between empirical distributions of the target energy $E(x)$. On $\mathbb{R}$ the optimisation in \cref{eq:w2} reduces to sorting both samples and matching them in order; we compute it using the POT package~\citep{flamary_pot_2021} and report $\sqrt{\mathcal{W}_2^2}$. For variance estimation we partition the MCMC reference into $10$ disjoint batches of $1\,000$, draw matched batches from each sampler (weighted draws without replacement for NTNS, uniform shuffling for the diffusion-based baselines), and report the mean and standard deviation across batches. The $E$-$\mathcal{W}_2$ values in \cref{tab:w2_metrics} use the same definition as in AS~\citep{havens_adjoint_2025} and ASBS~\citep{liu_adjoint_2025}, allowing extended comparisons with the numbers reported in those papers.

\section{Implementation Details}\label{sec:appendix}

\subsection{EGNN Architecture}\label{sec:egnn_details}

The velocity network used for particle systems is an E($n$)-equivariant graph neural network (EGNN)~\citep{satorras_en_2021}, consisting of $L$ stacked Equivariant Graph Convolutional Layers (EGCLs) operating on a fully-connected particle graph. Each node $i$ corresponds to a particle with position $\mathbf{x}_i \in \mathbb{R}^3$ and learned features $\mathbf{h}_i \in \mathbb{R}^{n_h}$. The flow time $t$ is broadcast as a scalar to every node and embedded into the hidden width by a single linear layer; no Fourier or sinusoidal time encoding is used. All MLPs in the network ($\phi_e, \phi_x, \phi_h$) use the SiLU activation.

Each EGCL layer updates node positions and features via edge messages. The edge model computes messages from pairwise features:
\begin{equation}\label{eq:egnn_edge}
    \mathbf{m}_{ij} = \phi_e(\mathbf{h}_i, \mathbf{h}_j, \|\mathbf{r}_{ij}\|^2),
\end{equation}
where $\mathbf{r}_{ij} = \mathbf{x}_i - \mathbf{x}_j$. Coordinate updates use the damped-normalized displacement $\hat{\mathbf{r}}_{ij} = \mathbf{r}_{ij} / (\|\mathbf{r}_{ij}\| + 1)$, following \citet{akhound-sadegh_iterated_2024}, and the coordinate MLP $\phi_x$ uses a small-scale kernel initialization ($\sigma^2 = 10^{-3}$, no bias) so the initial velocity field is near-identity. Residual connections are used in the node update, $\mathbf{h}_i' = \mathbf{h}_i + \phi_h(\mathbf{h}_i, \sum_j \mathbf{m}_{ij})$, and the network output is the accumulated displacement $v_\phi(\mathbf{x}, t) = \mathbf{x}^{(L)} - \mathbf{x}^{(0)}$ after mean-subtraction to remove any net center-of-mass drift, yielding an E($n$)-equivariant zero-mean velocity field.

\subsection{Hyperparameters}\label{sec:hyperparameters}

\Cref{tab:hyperparams} lists the NTNS hyperparameters used across all benchmarks; values are held fixed except where noted in the ablations of \cref{sec:ablations}.

\begin{table}[h]
\caption{Hyperparameters for NTNS experiments.}
\label{tab:hyperparams}
\centering
\begin{tabular}{lccc}
\toprule
\textbf{Parameter} & \textbf{DW-4} & \textbf{LJ-13} & \textbf{LJ-55} \\
\midrule
\multicolumn{4}{l}{\textit{Nested sampling}} \\
Live points $m$ & 2048 & 2048 & 2048 \\
Deletion / chains $k$ & 1024 & 1024 & 1024 \\
Inner steps $T$ & 8 & 39 & 165 \\
Initial prior oversample factor & 100 & 100 & 100 \\
Retrain interval & every iteration & every iteration & every iteration \\
\midrule
\multicolumn{4}{l}{\textit{MALA}} \\
Drift time $t^*$ & 0.5 & 0.5 & 0.5 \\
Initial step size $\epsilon_0$ & $10^{-3}$ & $10^{-3}$ & $10^{-3}$ \\
Target acceptance $\alpha$ & 0.3 & 0.3 & 0.3 \\
Gain $\gamma$ (\cref{eq:rm_step}) & 0.5 & 0.5 & 0.5 \\
Decay exponent $\kappa$ (\cref{eq:rm_step}) & 0 & 0 & 0 \\
\midrule
\multicolumn{4}{l}{\textit{Flow training (CFM)}} \\
Epochs per retrain & 150 & 150 & 150 \\
Batch size & 512 & 512 & 512 \\
Learning rate & $10^{-3}$ & $10^{-3}$ & $10^{-3}$ \\
Warm start & Yes & Yes & Yes \\
\midrule
\multicolumn{4}{l}{\textit{Velocity network (EGNN)}} \\
Layers $L$ & 5 & 5 & 5 \\
Hidden features $n_h$ & 128 & 128 & 128 \\
Activation & SiLU & SiLU & SiLU \\
\bottomrule
\end{tabular}
\end{table}

\subsection{Algorithm Pseudocode}\label{sec:algorithms}

We provide detailed pseudocode for the NTNS algorithm. \Cref{alg:ns_outer} describes one nested sampling outer iteration, while \Cref{alg:flow_mala} details the learned drift MALA kernel that NTNS uses to instantiate $\mathcal{K}_{E^*}$.

\begin{algorithm}[h]
\caption{Neural Transport Nested Sampling. Outer kernel structure follows \citet{yallup_nested_2026}; NTNS specialises the mutation kernel via a flow trained online on the live points.}
\label{alg:ns_outer}
\KwIn{Prior $\pi$; energy $E$; live count $m$; batch size $k$; flow class $Q_\phi$; mutation steps $T$; convergence threshold $\delta$}
\KwOut{Dead-point trajectory $\{(x, E, w(\beta))\}$ recovering $\mathcal{Z}(\beta)$ at any $\beta$}
Initialise $\{x_j\}_{j=1}^{m} \sim \pi$\;
\While{$\log \hat{\mathcal{Z}}_\text{live} - \log \hat{\mathcal{Z}} > \delta$}{
    $\mathcal{D} \gets \{x_j : E(x_j) \text{ in top-}k\}$,\quad $E^* \gets \min_{x \in \mathcal{D}} E(x)$ \tcp*{Threshold}
    Record $\mathcal{D}$ as dead points with per-point prior volumes and weights $w(\beta)$ from dynamic live count shrinkage~\citep{fowlie_nested_2021,yallup_nested_2026} (\cref{eq:weights}) \tcp*{Reweight}
    Train $Q_\phi$ on the $m-k$ survivors via \cref{eq:cfm_loss} (warm-started) \tcp*{Train}
    \For{$j = 1, \ldots, k$}{
        Sample parent $x \sim \mathrm{Uniform}(\{x_\ell\} \setminus \mathcal{D})$ \tcp*{Resample}
        $x_j^\text{new} \gets \mathcal{K}_{E^*}^{Q_\phi}(x;\,T)$ \tcp*{Mutate}
    }
    Replace $\mathcal{D}$ with $\{x_j^\text{new}\}_{j=1}^{k}$\;
}
\end{algorithm}

The learned drift MALA kernel exposes a single tuneable, the Langevin step size $\epsilon$, which we adapt between outer iterations of \cref{alg:ns_outer} via the standard stochastic update on the empirical acceptance $\bar a_i$ pooled across the $k$ mutation chains and $T$ inner steps,
\begin{equation}\label{eq:rm_step}
    \log \epsilon_{i+1} \gets \log \epsilon_i + \gamma_i (\bar a_i - \alpha),\qquad \gamma_i = \gamma / i^{\kappa},
\end{equation}
with target acceptance $\alpha$, gain $\gamma$ and decay exponent $\kappa$ ($\kappa = 0$ recovers a constant-gain/dual-averaging variant; $\kappa \in (\tfrac{1}{2},1]$ gives the classical stochastic approximation schedule). The kernel itself is the per-chain Langevin loop:

\begin{algorithm}[h]
\caption{Learned drift MALA kernel $\mathcal{K}_{E^*}^{Q_\phi}$, called inside \cref{alg:ns_outer}; step size $\epsilon$ is adapted between calls via \cref{eq:rm_step}.}
\label{alg:flow_mala}
\KwIn{Parent $x$; flow velocity $v_\phi$; energy threshold $E^*$; step size $\epsilon$; mutation steps $T$; drift time $t^*$}
\KwOut{Mutated sample $x$}
\For{$\ell = 1, \ldots, T$}{
    $\xi \sim \mathcal{N}(0, I)$, projected to zero centre-of-mass for particle systems\;
    $x' \gets x + \epsilon\, v_\phi(x, t^*) + \sqrt{2\epsilon}\,\xi$ \tcp*{Langevin proposal}
    $\alpha \gets \min\!\left(1,\; \dfrac{\pi(x')\, q(x \mid x')}{\pi(x)\, q(x' \mid x)}\right)$ with $q(x' \mid x) = \mathcal{N}\!\big(x';\, x + \epsilon\, v_\phi(x, t^*),\, 2\epsilon I\big)$\;
    \If{$u \sim \mathrm{Uniform}(0,1) < \alpha$ \textbf{and} $E(x') < E^*$}{
        $x \gets x'$ \tcp*{Accept}
    }
}
\Return $x$\;
\end{algorithm}

Key implementation details: (1) The convergence criterion $\log \hat{\mathcal{Z}}_\text{live} - \log\hat{\mathcal{Z}} < \delta$ (typically $\delta = -3$) estimates when remaining live point contribution is negligible. (2) The drift time $t^* = 0.5$ evaluates the velocity at the midpoint of the flow. (3) For particle systems, projecting the Gaussian noise and EGNN drift to zero centre-of-mass keeps the chain on the subspace $S$ (\cref{eq:com_subspace}); the EGNN's mean-subtracted velocity (\cref{eq:egnn_edge} ff.) preserves $S$ by construction. (4) We use $T = d$ Langevin steps per replacement, where $d$ is the dimensionality.

\begin{figure}[h]
    \centering
    \begin{subfigure}[t]{0.48\columnwidth}
        \centering
        \includegraphics[width=\linewidth]{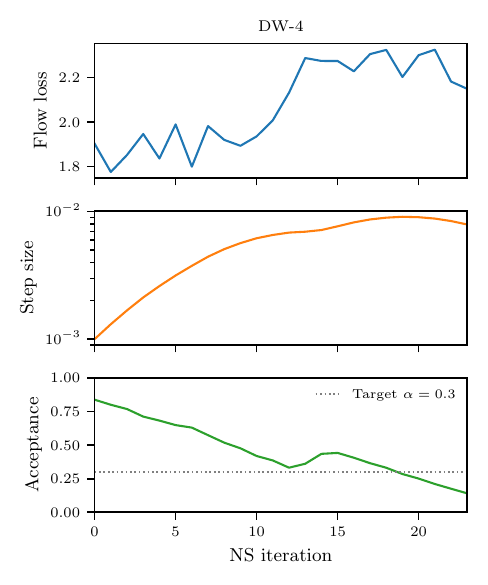}
        \caption{DW-4.}\label{fig:diagnostics:dw4}
    \end{subfigure}\hfill
    \begin{subfigure}[t]{0.48\columnwidth}
        \centering
        \includegraphics[width=\linewidth]{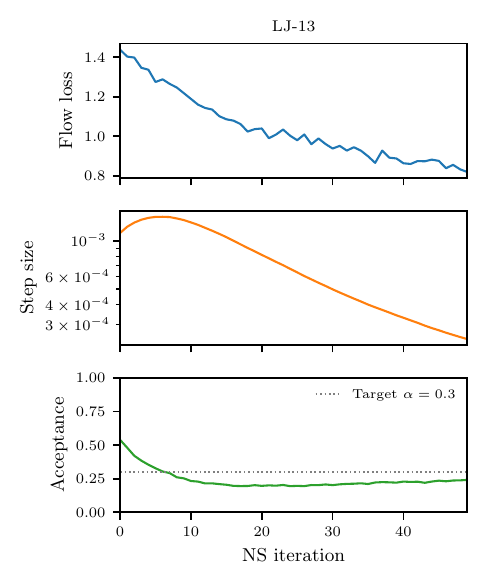}
        \caption{LJ-13.}\label{fig:diagnostics:lj13}
    \end{subfigure}
    \caption{MALA adaptation diagnostics for DW-4 and LJ-13. The acceptance rate and adapted step size remain stable across nested sampling iterations on both systems.}\label{fig:dw4_lj13_diagnostics}
\end{figure}

\section{Proof of Correctness}\label{sec:proof_correctness}

\begin{proof}[Proof of \cref{prop:correctness}]
Let $\pi^*(x) \propto \pi(x) \mathbf{1}[E(x) < E^*]$ denote the constrained prior. For particle systems we read all densities below as densities on the zero-CoM subspace $S$ of \cref{eq:com_subspace}: the noise in \cref{eq:mala_proposal} is projected onto $S$ and the EGNN drift is zero-CoM by construction, so $q(\cdot \mid x)$ is the (non-degenerate) Gaussian transition density on $S$ for $x \in S$ and the chain remains on $S$. The learned drift MALA kernel proposes $x' \sim q(\cdot \mid x)$ from \cref{eq:mala_proposal} and accepts with probability $\alpha(x, x')$ from \cref{eq:mala_accept}.

For any states $x, x'$ both satisfying the energy constraint, the acceptance probability satisfies
\begin{equation}
    \alpha(x, x') = \min\left(1, \frac{\pi(x') q(x \mid x')}{\pi(x) q(x' \mid x)}\right).
\end{equation}
The detailed balance condition requires $\pi^*(x) q(x' \mid x) \alpha(x, x') = \pi^*(x') q(x \mid x') \alpha(x', x)$.

\emph{Case 1.} If $\pi(x') q(x \mid x') \geq \pi(x) q(x' \mid x)$, then $\alpha(x, x') = 1$ and
\begin{equation}
    \alpha(x', x) = \frac{\pi(x) q(x' \mid x)}{\pi(x') q(x \mid x')}.
\end{equation}
Thus the LHS equals $\pi(x) q(x' \mid x)$ and the RHS equals $\pi(x') q(x \mid x') \cdot \frac{\pi(x) q(x' \mid x)}{\pi(x') q(x \mid x')} = \pi(x) q(x' \mid x)$.

\emph{Case 2.} The symmetric case follows analogously.

For proposals violating the constraint ($E(x') \geq E^*$), $\alpha = 0$, so $x' \notin \text{supp}(\pi^*)$ is never reached from $x \in \text{supp}(\pi^*)$. This preserves detailed balance on $\text{supp}(\pi^*)$.

Since detailed balance holds with respect to $\pi^*$, it is a stationary distribution. Under mild conditions (the Gaussian proposal $q$ having positive density everywhere), the chain is irreducible and aperiodic on the constrained region, so $\pi^*$ is the unique stationary distribution. Crucially, this argument depends only on $q$ being a valid proposal density---the flow parameters $\phi$ affect the drift direction but not the validity of the MH correction.
\end{proof}

\section{Ablations}\label{sec:ablations}
\begin{table}[h]
\caption{Ablation results on LJ-13. Single-axis perturbations to NTNS on the
  $d{=}39$ Lennard-Jones target; the top row is the baseline, each block varies one hyperparameter. $r$-$W_2$ / $E$-$W_2$ are mean$\pm$std of the $1$D Wasserstein-2 over $10$ batches of $1000$ NS-weighted draws vs.\ permuted MCMC reference batches; runtime is wall time on one GH200.}
\centering
\resizebox{\textwidth}{!}{%
\begin{tabular}{llrrrr}
\toprule
Ablation & Setting & $r$-$W_2$ & $E$-$W_2$ & Runtime (s) & Evals (M) \\
\midrule
Baseline & - & $0.026 \unc{0.006}$ & $0.605 \unc{0.129}$ & 583 & 10.15 \\
\midrule
Inner steps $T$ & $19$ & $0.031 \unc{0.006}$ & $0.476 \unc{0.077}$ & 598 & 5.05 \\
 & $78$ & $0.021 \unc{0.008}$ & $0.541 \unc{0.185}$ & 612 & 20.17 \\
\midrule
Drift time $t^*$ & $0.1$ & $0.031 \unc{0.006}$ & $0.570 \unc{0.141}$ & 607 & 10.19 \\
 & $0.25$ & $0.027 \unc{0.005}$ & $0.895 \unc{0.168}$ & 617 & 10.15 \\
 & $0.5$ & $0.044 \unc{0.009}$ & $1.078 \unc{0.232}$ & 606 & 10.19 \\
 & $0.75$ & $0.039 \unc{0.006}$ & $0.804 \unc{0.147}$ & 595 & 10.19 \\
 & $0.9$ & $0.020 \unc{0.006}$ & $0.589 \unc{0.114}$ & 597 & 10.11 \\
\midrule
Flow epochs $E_f$ & $75$ & $0.033 \unc{0.004}$ & $1.045 \unc{0.148}$ & 464 & 10.19 \\
 & $300$ & $0.021 \unc{0.004}$ & $0.774 \unc{0.135}$ & 859 & 10.11 \\
\midrule
Live / deletion & $m{=}4096,\ k{=}1024$ & $0.018 \unc{0.007}$ & $0.479 \unc{0.125}$ & 2071 & 21.18 \\
 & $m{=}4096,\ k{=}2048$ & $0.012 \unc{0.006}$ & $0.549 \unc{0.123}$ & 851 & 20.30 \\
\bottomrule
\end{tabular}}
\end{table}

We ablate four single-axis perturbations of NTNS on LJ-13, holding all other hyperparameters at their values in \cref{tab:hyperparams}; runtimes are wall-clock on a single GH200.

Drift time $t^*$ has only a mild effect on sample quality: $r$-$W_2$ varies between 0.020 and 0.044 across the swept range, with the best result at $t^* = 0.9$, suggesting that fixing $t^*$ to a higher value is a sensible default to investigate further. Flow epochs $E_f$ are the dominant handle on runtime, scaling close to linearly: halving to $E_f = 75$ cuts runtime to 464\,s and doubling to $E_f = 300$ increases it to 859\,s, with the deviation from exact proportionality due to fixed-cost overhead. Inner kernel steps $T$ are the dominant handle on target-evaluation count: halving $T$ to $19$ halves the eval count to $\approx 5{\times}10^6$ at only a small cost in $r$-$W_2$ (0.031 vs.\ the baseline 0.026), an important saving for more expensive energy functions where our default $T = d = 39$ is conservatively high. The live set / deletion choice is the most sensitive handle overall: increasing $m$ to $4096$ while holding $k = 1024$ slows the compression rate (yielding a more accurate $\log\mathcal{Z}$, cf.\ \citet{yallup_nested_2026}) but more than triples runtime to 2071\,s; doubling both $m$ and $k$ to $(4096, 2048)$ preserves the baseline compression rate and recoups most of this cost at 851\,s. Both increase per-iteration target evaluations because the number of replacement chains scales with $k$. Across the four axes, sample quality stays robust while runtime and evaluation cost respond predictably to the relevant handles, suggesting a stable underlying algorithm. 

Within this, the flow training loop emerges as the largest overall bottleneck and the most promising target for future optimisation. The $T \in \{T/2,\,T,\,2T\}$ sweep above is the same numerical-accuracy check used for Markov-move budgets in SMC; we recommend running this short stability check on a new target and verifying that posterior diagnostics and, where relevant, $\log\mathcal{Z}(\beta)$ are stable within Monte Carlo variation before committing compute at the production budget.

\subsection{Inner-kernel target acceptance}\label{sec:ablation_target_accept}

The step size update of \cref{eq:rm_step} adapts $\log\epsilon$ toward a target acceptance $\alpha$. For smooth log-concave targets in the high-dimensional limit, the optimal MALA target is $\alpha \approx 0.574$~\citep{roberts_optimal_1998}. We sweep two points on LJ-55, $\alpha = 0.5$ (closer to this textbook canonical) and $\alpha = 0.3$, holding every other hyperparameter fixed and running five seeds at each setting; the seed-aggregate metrics are reported in \cref{tab:seeds_w2}.

The lower target wins on every measure: better fidelity to the MCMC reference on both $r$- and $E$-$W_2$, and a $\sim$3$\times$ tighter cross-seed scatter on $\log\mathcal{Z}(1)$ that brings the TI cross-check from $3.4\sigma$ to $1.5\sigma$ of the NUTS reference (\cref{sec:phase_appendix}). We attribute the departure from the $0.574$ asymptote to the constrained prior structure of NS: the inner kernel targets $\pi^*_i \propto \pi \cdot \mathbf{1}[E < E^*_i]$, whose support has a hard edge that the smooth log-concave analysis underlying \citet{roberts_optimal_1998} does not capture. A lower target acceptance drives the adapted $\epsilon$ to a larger value (mean $\epsilon \approx 2.1{\times}10^{-4}$ at $\alpha = 0.3$ vs.\ $9.7{\times}10^{-5}$ at $\alpha = 0.5$), so proposals more frequently strike the energy boundary and the chain accumulates better coverage of the constraint surface in the same wall-clock budget. We therefore adopt $\alpha = 0.3$ as the canonical setting for the LJ-55 experiments reported elsewhere.

\subsection{Classical nested sampling control}\label{sec:nss_control}

To isolate the contribution of the learned mutation kernel from that of the nested sampling outer loop, we compare NTNS with vectorized Nested Slice Sampling (NSS)~\citep{yallup_nested_2026}. NSS uses a non-learned joint hit-and-run slice sampler for constrained prior replacement. We use the same LJ-55 target and the same $m=2048$, $k=1024$ live-point and deletion schedule as NTNS.

\begin{table}[h]
\centering
\caption{Matched LJ-55 comparison between NTNS and classical NSS. Values are mean $\pm$ standard deviation across independent runs. Lower $W_2$ is better. Runtime is wall time on one GH200; evaluations count target-energy calls.}
\label{tab:nss_control}
\resizebox{\textwidth}{!}{%
\begin{tabular}{lccccccc}
\toprule
Method & $m/k$ & Runs & $r$-$W_2$ $\downarrow$ & $E$-$W_2$ $\downarrow$ & $\log\mathcal{Z}(1)$ & Time (s) & Evals \\
\midrule
NTNS & $2048/1024$ & $5$ & $\mathbf{0.017} \unc{0.020}$ & $\mathbf{3.97} \unc{2.03}$ & $233.85 \unc{0.83}$ & $19{,}541$ & $8.0{\times}10^7$ \\
NSS  & $2048/1024$ & $10$ & $0.066 \unc{0.052}$ & $25.74 \unc{19.37}$ & $230.7 \unc{11.6}$ & $144 \unc{15}$ & $4.82{\times}10^8$ \\
\bottomrule
\end{tabular}%
}
\end{table}

Under the matched outer schedule, NTNS gives substantially lower distance and energy errors while NSS uses approximately six times more target evaluations. NSS is faster in wall time because the analytic Lennard--Jones energy is exceptionally cheap, whereas NTNS runtime is dominated by flow training; we therefore report both resources. Increasing NSS to $m=4096$, $k=2048$ does not establish convergence: its cross-run estimate shifts from $\log\mathcal{Z}(1)=230.7\unc{11.6}$ to $243.0\unc{5.8}$. The controlled comparison shows that the improvement arises from the learned mutation dynamics rather than the nested sampling outer loop alone.

\section{Convergence and Reproducibility}\label{sec:convergence}

A standard nested sampling termination criterion is to stop once the remaining live set contribution to the partition function is negligible---typically when $\log\hat{\mathcal{Z}}_\text{live} - \log\hat{\mathcal{Z}} < -3$~\citep{skilling_nested_2006,handley_polychord_2015}. This implicitly targets $\beta = 1$ and is a sensible default; the threshold can be tuned to target other temperatures or to budget runtime.

This principled stopping rule is a structural advantage of NTNS over the diffusion-based baselines. iDEM, AS, and ASBS run for a fixed wall-clock budget with no internal indication of having reached the true distribution; on LJ-55 in particular all three timed out (\cref{tab:timing}), and the metrics we report for AS and ASBS at the time-out point are consistent with the extended comparisons in \citet{liu_adjoint_2025}. Across the seed sweep (\cref{fig:seeds}), NTNS is consistently the best-performing sampler, and continued training of the diffusion baselines does not drive them to a stable answer.

All methods can suffer training instabilities. NTNS provides a clean response. The target acceptance $\alpha$ interacts with this directly: at $\alpha = 0.5$ the adapted $\epsilon$ is small enough that the chain occasionally drives acceptance to zero, at which point we trigger a restart of the inner kernel from the current live set with $\epsilon$ reset to its initial value (\cref{fig:lj55_seeds:alpha05}). At our canonical $\alpha = 0.3$ this collapse mode does not occur across any of the five seeds we ran (\cref{fig:lj55_seeds:alpha03}); the chain occasionally produces a batch with non-finite flow-training gradients, which we apply a simple online filter to. The forward-pass loss is still logged on those batches, so the reported flow loss can spike, but the trained network and the chain's operating point are unaffected and this does not show up as a difference in across-seed sample quality (\cref{tab:seeds_w2}). Good initialisation helps NTNS too, but the algorithm is markedly more stable than the alternatives.

\begin{table}[h]
\centering
\caption{Computational cost comparison; all runs use a single NVIDIA GH200 GPU. NTNS achieves substantially the lowest walltime across all baselines, while AS / ASBS use fewer energy queries by design---their adjoint-matching objective only requires $\nabla E$ at trajectory endpoints. Every iDEM, AS and ASBS evaluation is differentiated to obtain $\nabla E$; NTNS evaluations are of $E$ alone. Entries marked $^{*}$ reached the 12\,h training time limit before convergence.}
\label{tab:timing}
\begin{tabular}{lcccccc}
\toprule
& \multicolumn{2}{c}{DW-4} & \multicolumn{2}{c}{LJ-13} & \multicolumn{2}{c}{LJ-55} \\
\cmidrule(lr){2-3} \cmidrule(lr){4-5} \cmidrule(lr){6-7}
Method & Time (s) & Evals & Time (s) & Evals & Time (s) & Evals \\
\midrule
iDEM & 9,633 & $1{\times}10^8$ & 19,877 & $1{\times}10^8$ & $43{,}200^{*}$ & $6{\times}10^7$ \\
AS   & 7,260 & $4{\times}10^6$ & $43{,}200^{*}$ & $4{\times}10^6$ & $43{,}200^{*}$ & $\mathbf{3{\times}10^5}$ \\
ASBS & 6,960 & $4{\times}10^6$ & 34,680 & $\mathbf{3{\times}10^6}$ & $43{,}200^{*}$ & $\mathbf{3{\times}10^5}$ \\
\midrule
NTNS & \textbf{233} & $\mathbf{2{\times}10^6}$ & \textbf{583} & $1{\times}10^7$ & \textbf{19,541} & $8{\times}10^7$
\\
\bottomrule
\end{tabular}
\end{table}

\begin{figure}[h]
    \centering
    \begin{subfigure}[t]{\columnwidth}
        \centering
        \includegraphics[width=\linewidth]{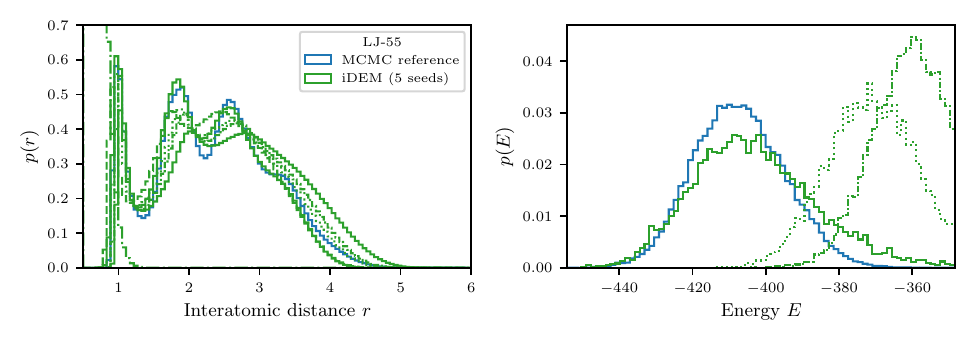}
        \caption{iDEM results showing consistent deviation from reference.}\label{fig:seeds:idem}
    \end{subfigure}\\
    \begin{subfigure}[t]{\columnwidth}
        \centering
        \includegraphics[width=\linewidth]{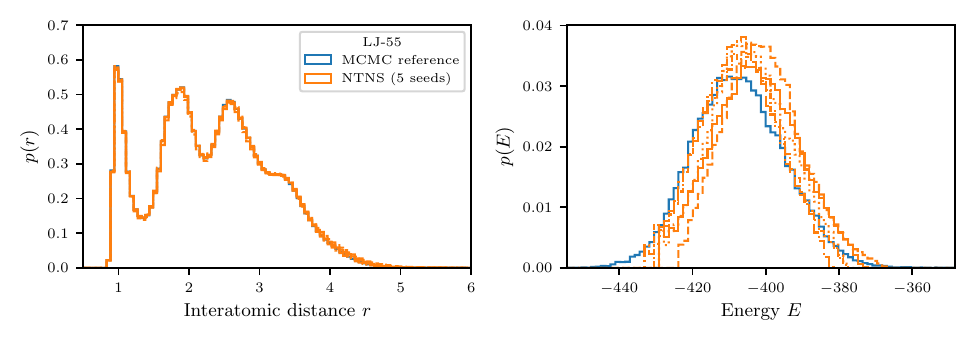}
        \caption{NTNS results showing consistent agreement with reference.}\label{fig:seeds:ntns}
    \end{subfigure}
    \caption{Repeated LJ-55 runs over 5 consecutive seeds. iDEM (top) shows consistent deviation from the MCMC reference, whereas NTNS (bottom) reproduces it across all seeds.}\label{fig:seeds}
\end{figure}

\begin{table}[h]
\centering
\caption{Wasserstein-2 distances to MCMC reference and partition function estimate for the LJ-55 seeded runs in \cref{fig:seeds}, reported as mean $\pm$ std across the five seeds at each setting. Lower $W_2$ is better. The two NTNS rows differ only in the target acceptance $\alpha$ (\cref{sec:ablation_target_accept}); $\alpha = 0.3$ is our canonical setting. The single-run LJ-55 entries in the main \cref{tab:w2_metrics} for both NTNS and iDEM are median-seed runs drawn from this table: NTNS ($r$-$W_2 = 0.008$, $E$-$W_2 = 2.68$) and iDEM ($r$-$W_2 = 0.135$, $E$-$W_2 = 48.4$). The same NTNS run is shown in \cref{fig:lj_diagnostics:lj55} and \cref{fig:phase_diagram}. One of the iDEM trainings diverged so is omitted. $^{*}$Reference value taken from \citet{akhound-sadegh_iterated_2024}.}
\label{tab:seeds_w2}
\begin{tabular}{lccc}
\toprule
Method & $r$-$W_2$ $\downarrow$ & $E$-$W_2$ $\downarrow$ & $\log\mathcal{Z}(1)$ \\
\midrule
iDEM & $0.136 \unc{0.093}$ & $78.6 \unc{75.3}$ & $273.2 \unc{22.2}^{*}$ \\
NTNS ($\alpha = 0.5$) & $0.041 \unc{0.024}$ & $8.4 \unc{3.5}$ & $229.5 \unc{2.2}$ \\
NTNS ($\alpha = 0.3$) & $\mathbf{0.017} \unc{0.020}$ & $\mathbf{3.97} \unc{2.03}$ & $233.85 \unc{0.83}$ \\
\bottomrule
\end{tabular}
\end{table}

\begin{figure}[h]
    \centering
    \begin{subfigure}[t]{0.48\columnwidth}
        \centering
        \includegraphics[width=\linewidth]{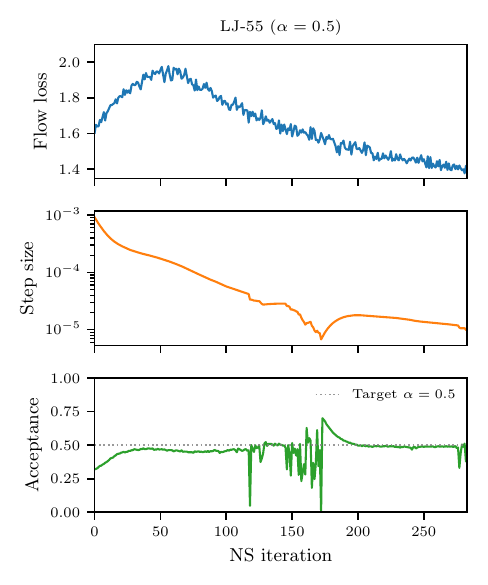}
        \caption{$\alpha = 0.5$.}\label{fig:lj55_seeds:alpha05}
    \end{subfigure}\hfill
    \begin{subfigure}[t]{0.48\columnwidth}
        \centering
        \includegraphics[width=\linewidth]{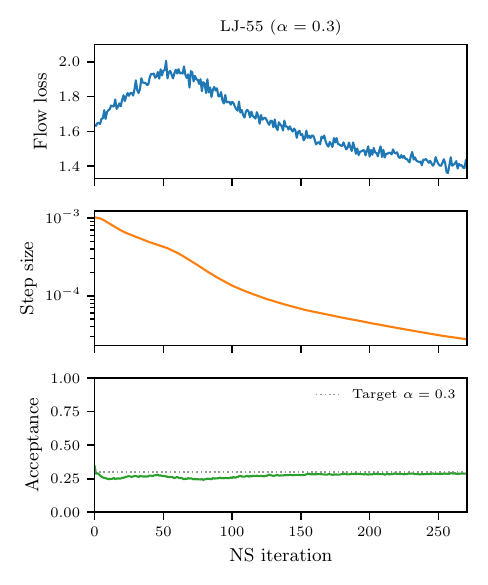}
        \caption{$\alpha = 0.3$.}\label{fig:lj55_seeds:alpha03}
    \end{subfigure}
    \caption{LJ-55 MALA diagnostics at two target acceptances $\alpha$. Each panel shows the per-iteration flow training loss, adapted Langevin step size $\epsilon$, and acceptance rate; the dotted line in the acceptance panel marks $\alpha$. At $\alpha = 0.5$ (\cref{fig:lj55_seeds:alpha05}) the acceptance occasionally collapses to zero and is recovered by a restart from the current live set with $\epsilon$ reset. At $\alpha = 0.3$ (\cref{fig:lj55_seeds:alpha03}) the acceptance stays well above zero throughout, never triggering a restart.}\label{fig:lj55_seeds_diagnostics}
\end{figure}

\section{Phase-Diagram Analysis}\label{sec:phase_appendix}

We perform an independent check of NTNS's partition function estimate on LJ-55 via thermodynamic integration~\citep{kirkwood_statistical_1935} (TI) over $\beta \in [0.6, 1.0]$. The TI identity
\begin{equation}\label{eq:ti}
    \log \mathcal{Z}(\beta_2) - \log \mathcal{Z}(\beta_1) \;=\; -\int_{\beta_1}^{\beta_2} \langle E \rangle_\beta\,\mathrm{d}\beta,
\end{equation}
relates differences in $\log \mathcal{Z}$ to averaged energies. We ran NUTS~\citep{hoffman_no-u-turn_2014} chains at $\beta \in \{0.6, 0.8, 1.0\}$ targeting $\exp(-\beta E(x))\,\pi(x)$ and applied three-point Simpson quadrature to $\langle E \rangle_\beta$, giving $\log \mathcal{Z}(1.0) - \log \mathcal{Z}(0.6) = 148.8 \pm 0.4$. The corresponding NTNS estimate, from the dynamic live count posterior over the prior volume sequence~\citep{fowlie_nested_2021,yallup_nested_2026}, is $\log \mathcal{Z}(1.0) - \log \mathcal{Z}(0.6) = 147.54 \pm 0.44$, with $\log \mathcal{Z}(1.0) = 234.65 \pm 0.34$; across the five reproducibility seeds in \cref{fig:seeds}, the mean and across-seed standard deviation are $147.75 \pm 0.73$ and $\log\mathcal{Z}(1.0) = 233.85 \pm 0.83$ (\cref{tab:seeds_w2}).

\begin{figure}[t]
    \centering
    \includegraphics[width=0.48\columnwidth]{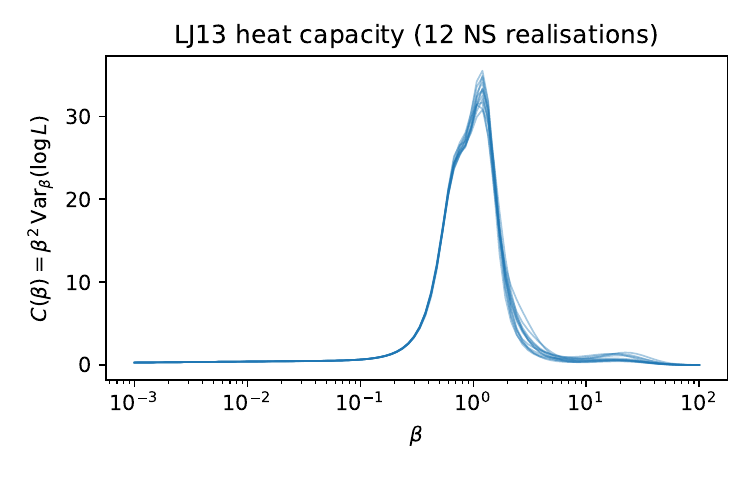}\hfill
    \includegraphics[width=0.48\columnwidth]{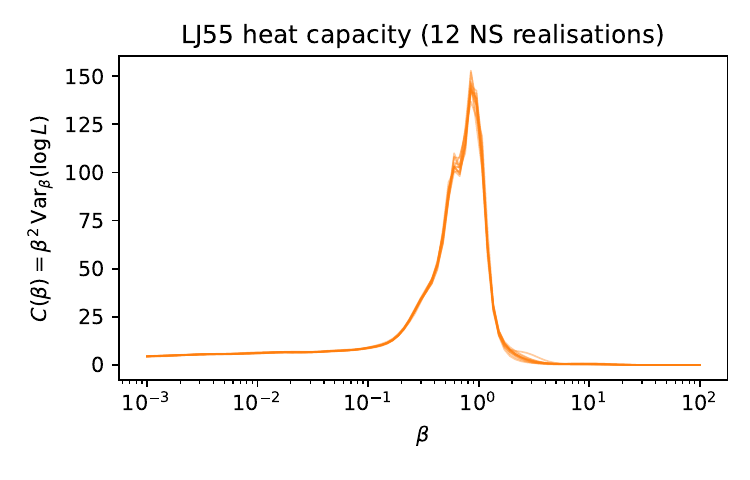}
    \caption{Configurational heat capacity $C(\beta) = \beta^2\,\mathrm{Var}_\beta(\log\mathcal{L})$ for LJ-13 (left) and LJ-55 (right), reweighted from a single NTNS run across the temperatures supported by the dead point trajectory; overlaid traces show the variability across simulated prior volume sequences~\citep{handley_polychord_2015}. Both curves peak near $\beta \approx 1$, which we interpret as a cluster formation transition along the likelihood-tempering path under the Gaussian prior.}\label{fig:heat_capacity}
\end{figure}

\Cref{fig:heat_capacity} reports the configurational heat capacity $C(\beta) = \beta^2\,\mathrm{Var}_\beta(\log\mathcal{L})$ along the tempering path, derived from the same dead point trajectory. We use the Gaussian prior Boltzmann convention $\mu_\beta(x) \propto \pi(x)\exp[-\beta E(x)]$ with $\pi(x) = \mathcal{N}(0, I)$ established by~\citet{klein_equivariant_2023} and adopted by all neural samplers considered in this work; this differs from the molecular dynamics LJ tradition, which uses hard-sphere confinement at cluster-specific radii. Across simulated prior volume sequences~\citep{handley_polychord_2015}, the per-atom peak heights are $C_\text{peak}/N = 2.692 \pm 0.089$ at $\beta_c = 1.169 \pm 0.053$ for LJ-13 and $C_\text{peak}/N = 2.715 \pm 0.038$ at $\beta_c = 0.875 \pm 0.032$ for LJ-55.

The LJ literature reports a melting transition at $\beta \approx 3.4$ ($T^* \approx 0.295$) for LJ-55~\citep{frantz_2001,partay_nested_2014,partay_nested_2021} in the standard $r_m$-form convention. The benchmark energy function we use differs from the more standard convention by a factor of two~\citep{klein_equivariant_2023,akhound-sadegh_iterated_2024,havens_adjoint_2025}, so the literature melting at $\beta \approx 3.4$ maps to $\beta \approx 1.7$ in our convention. Independently, the classical literature uses hard-sphere confinement, whereas the Gaussian prior provides softer confinement that broadens and shifts the transition. After conversion, our LJ-55 peak at $\beta_c = 0.875$ lies below the converted literature melting, consistent with a softened transition along the Gaussian prior tempering path. The two per-atom peak heights also agree to within $1\%$, consistent with approximately additive cluster formation rather than a sharply collective finite-size transition.

This convention difference means the standard neural sampler benchmark target is somewhat softer than in the classical literature. Even so, NTNS is the only method that comes close on LJ-55, and we conjecture it has the best chance to scale to the harder classical target. In the athermal sampling procedure of NS this is straightforward: the conventional termination criterion $\log\mathcal{Z}_\text{live} - \log\mathcal{Z} < -3$ targets $\beta = 1$, and reaching $\beta = 2$ (the standard convention given this energy definition) only requires lowering it further.

\section{IMH Results and Scaling Analysis}\label{sec:irmh}

The independent Metropolis--Hastings (IMH) kernel (\cref{sec:flow_mala}) provides an alternative to MALA that uses the flow density directly in the acceptance ratio. Concretely, given a flow $Q_\phi$ trained on the current live points, the IMH kernel proposes $x' \sim Q_\phi$ by sampling $x_0 \sim p_0$ and integrating the forward ODE $x' = \psi_1(x_0)$, recording $\log Q_\phi(x')$ via the change-of-variables formula \cref{eq:change_of_vars}. The proposal is accepted with probability
\begin{equation}\label{eq:mh_ratio}
    \alpha = \min\!\left(1,\; \frac{\pi(x') / Q_\phi(x')}{\pi(x) / Q_\phi(x)}\right) \cdot \mathbf{1}[E(x') < E^*],
\end{equation}
where the indicator enforces the energy constraint and the importance ratio corrects for the mismatch between the flow and the constrained prior. The kernel leaves $\pi_i^*$ invariant for any $Q_\phi$, so the flow's quality affects only mixing efficiency. This is conceptually similar to the importance-sampling construction in nessai~\citep{williams_nested_2021,williams_importance_2023}, which corrects the same flow/target mismatch via importance weights and applies a manual weight cut to tame heavy-tailed weights; our IMH formulation handles those tails naturally through MH rejection, without requiring such a cut.

Unlike MALA, IMH requires integrating the ODE both forward (to generate proposals) and backward (to compute the density $Q_\phi(x)$ via \cref{eq:change_of_vars}). We use an adaptive Dormand--Prince 5(4) solver (dopri5)~\citep{prince_high_1981} at tolerance $10^{-3}$ for these integrations. Two cheaper alternatives are worth flagging. Single-step coupling flows (e.g.\ RealNVP-style architectures) admit cheap exact densities but are limited in expressivity for the multimodal targets we consider. Stochastic trace estimators such as Hutchinson reduce the per-step divergence cost in expectation, but inject noise into the log-density and therefore into the MH acceptance, complicating reliable correction~\citep{grenioux_sampling_2023}; we use exact divergence throughout this analysis. While IMH achieves comparable sample quality to MALA on LJ-13, its computational cost scales poorly with dimension due to the backward ODE required for density evaluation.

For LJ-13, the IMH inner kernel achieves $r$-$W_2 = 0.011 \unc{0.006}$ and $E$-$W_2 = 0.77 \unc{0.19}$, comparable to MALA ($r$-$W_2 = 0.021 \unc{0.007}$, $E$-$W_2 = 0.65 \unc{0.18}$) and both within statistical error of the MCMC-vs-MCMC control. The IMH kernel maintains acceptance rates of 15--37\% over 51 NS iterations (mean 26\%), completing in 4.4 hours at 3.3 points/second (\cref{fig:irmh}).

However, applying IMH to LJ-55 (165 dimensions) yielded acceptance rates below 1\% with prohibitive runtime. The backward ODE integration requires computing $\nabla \cdot v_\phi$ along the trajectory, which involves $\mathcal{O}(d)$ vector-Jacobian products per integration step. At $d = 165$, this cost dominates and the density estimates become less accurate, leading to poor acceptance. This motivated the development of the MALA variant (\cref{sec:flow_mala}), which avoids density evaluation entirely by using only the flow velocity as a drift term.

These results suggest that exact importance correction via backward ODE is practical for moderate dimensions ($d \lesssim 50$) but does not scale to higher-dimensional systems where MALA-style approaches are preferred.

\begin{figure}[h]
    \centering
    \begin{subfigure}[t]{0.66\columnwidth}
        \centering
        \includegraphics[width=\linewidth]{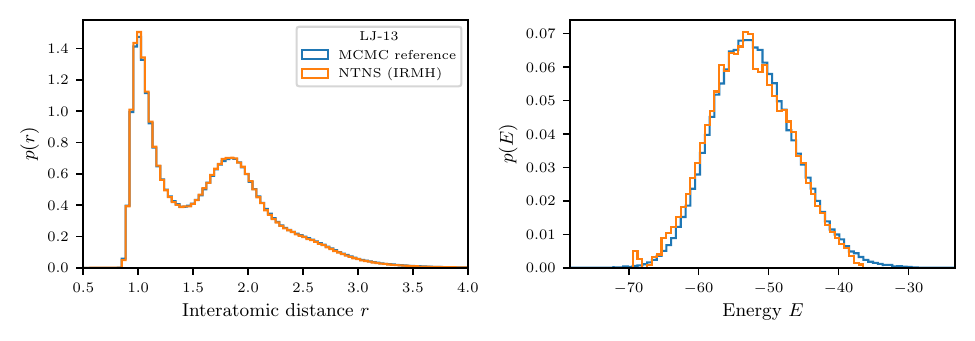}
        \caption{LJ-13 IMH posterior diagnostics showing interatomic distance and energy distributions compared to MCMC reference.}\label{fig:irmh:histograms}
    \end{subfigure}\hfill
    \begin{subfigure}[t]{0.33\columnwidth}
        \centering
        \includegraphics[width=\linewidth]{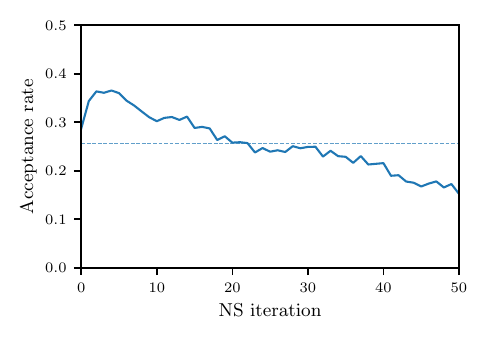}
        \caption{IMH acceptance rate across nested sampling iterations.}\label{fig:irmh:acceptance}
    \end{subfigure}
    \caption{IMH diagnostics on LJ-13. Left: pooled interatomic distance and energy histograms against the MCMC reference. Right: MH acceptance rate across nested sampling iterations. IMH attains comparable sample quality to MALA on LJ-13 but at substantially higher runtime.}\label{fig:irmh}
\end{figure}

\section{Nested Sampling vs SMC Paths}\label{sec:ns_vs_smc}

A common characteristic that differentiates sampling approaches is the path followed between the prior and posterior. A broad class of neural samplers construct a variational proposal using a neural density estimate; there is then a path that the ELBO minimisation follows~\citep{pmlr-v235-blessing24a}, which can succumb to mode-seeking behavior, and once trained the resulting model transports a sample from the known prior to the approximate posterior. This is where nested sampling clearly distinguishes itself: rather than following the noising path typical of denoising diffusion approaches, or the thermal path typified by annealing, nested sampling follows a unique path with useful properties.

Primarily, the nested sampling path is of interest for molecular modelling as it has unique advantages when sampling from distributions with \emph{phase transitions}~\citep{murray_nested_2005}. We illustrate an example continuous second-order phase transition in \Cref{fig:pt}. The nested sampling replacement kernel described in \cref{sec:ns_background} is characterized by a sorting operation on the particle cloud energy levels, which gives an automatic sequence of interpolating distributions. When approaching a flat plateau in energy as in the illustration, sorting remains a valid operation that steadily compresses the volume. This contrasts favourably with common bridging schemes such as annealing, which either experience weight collapse when trying to jump in temperatures as the critical temperature is approached, or critically slow down if the annealing step size is adapted.

The nested sampling path of distributions is unique, and provided there is an efficient way to draw samples from the constrained prior~\cref{eq:constrained_prior}, the algorithm retains its theoretical guarantees. Solving this sampling problem is where we employ flow matching models to improve efficiency.

\section{Supplementary Figures}\label{sec:supplementary_figures}

\begin{figure}[h]
  \centering
  \includegraphics[width=\columnwidth]{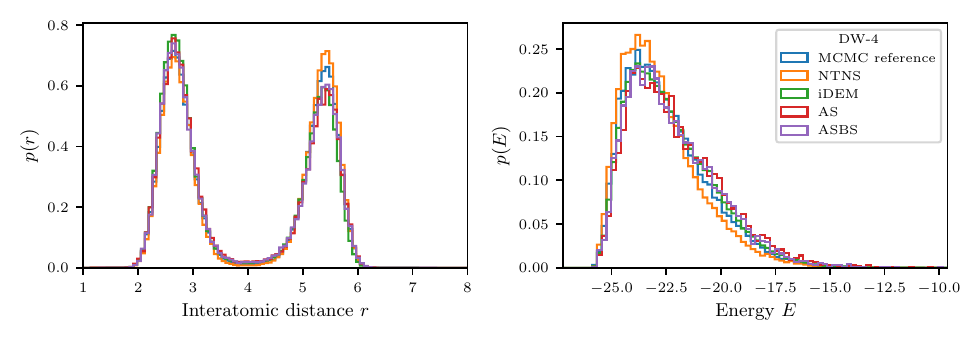}
  \caption{DW-4 posterior diagnostics comparing NTNS, iDEM, AS and ASBS with MCMC reference samples. Left: pooled interatomic distance distribution. Right: energy distribution.}\label{fig:dw4_diagnostics}
\end{figure}


\end{document}